\documentclass{article}

\usepackage{arxiv}

\usepackage{amsmath,amssymb,amsthm}

\newtheorem{theorem}{Theorem}
\newtheorem{lemma}{Lemma}
\newtheorem{definition}{Definition}

\newtheorem{remark}{Remark}

\usepackage[utf8]{inputenc} 
\usepackage[T1]{fontenc}    
\usepackage{hyperref}       
\usepackage{url}            
\usepackage{booktabs}       
\usepackage{amsfonts}       
\usepackage{nicefrac}       
\usepackage{microtype}      
\usepackage{cleveref}       
\usepackage{lipsum}         
\usepackage{graphicx}
\usepackage{natbib}
\usepackage{doi}

\usepackage{tikz}
\usetikzlibrary{arrows.meta, positioning, shapes.geometric, shapes.symbols, fit, backgrounds, calc}
\usepackage{adjustbox}
\usepackage{multirow}
\usepackage{subcaption}

\usepackage[ruled,vlined,linesnumbered]{algorithm2e}

\makeatletter

\renewcommand{\sectionautorefname}{\S\@gobble}
\renewcommand{\subsectionautorefname}{\S\@gobble}
\renewcommand{\subsubsectionautorefname}{\S\@gobble}
\def\appendixautorefname{\S\@gobble}%

\makeatother

\newtoggle{usecomment}
\settoggle{usecomment}{false}
\newcommand{\tl}[1]{\iftoggle{usecomment}{{\color{red}{[TL]: #1}}}{}}

\newcommand{\tool}{\textsc{ZonoGpt}}
\newcommand{\toolD}{\textsc{ZonoGpt-D}}
\newcommand{\toolA}{\textsc{ZonoGpt-A}}
\newcommand{\toolM}{\textsc{ZonoGpt-M}}
\newcommand{\toolF}{\textsc{ZonoGpt-F}}
\newcommand{\ibp}{\textsc{IBP}}
\newcommand{\deepz}{\textsc{DeepZ}}
\newcommand{\covenn}{\textsc{CoVeNN}}
\newcommand{\abcrown}{\textsc{$\alpha\beta$-CROWN}}
\newcommand{\R}{\mathbb{R}}

\newcommand{\eg}{\emph{e.g.}}

\title{\tool{}: Towards An Abstract Domain for Verifying Large GPT Models}

\author{Hai Duong \\
	Department of Computer Science \\
	George Mason University \\
        Fairfax, VA, USA \\
	\And
	Thanh Le \\
	Unaffiliated \\
	Yokosuka, Japan \\
	\And
	ThanhVu Nguyen \\
	Department of Computer Science \\
	George Mason University \\
        Fairfax, VA, USA \\
}

\date{}

\renewcommand{\undertitle}{}
\renewcommand{\shorttitle}{\tool{}: Towards An Abstract Domain for Verifying Large GPT Models}

\hypersetup{
  pdftitle={ZonoGpt: Towards An Abstract Domain for Verifying Large GPT Models},
  pdfsubject={q-bio.NC, q-bio.QM},
  pdfauthor={Hai Duong, Thanh Le, ThanhVu Nguyen},
  pdfkeywords={First keyword, Second keyword, More},
}

\begin{document}
\maketitle

\begin{abstract}
  Transformer-based models are widely used for reasoning, coding, and multimodal agentic tasks.
  To provide formal assurance of desirable behaviors, such as robustness, safety, and fairness,
  neural network verification techniques prove required properties and provide auditable guarantees before deployment.
  However, prior work remains limited to small or restricted Transformers, and maintaining precision across deep models remains challenging.
  In this work, we introduce \tool{}, an abstract domain for verifying large transformers that maintains a space complexity independent of network depth.
  \tool{} uses a structured zonotope and a generator reduction mechanism to efficiently preserve correlations.
  To maintain precision, it introduces block-specific fused transformations for Attention and LayerNorm that retain feature relations,
  along with an affine transform for GELU that preserves generator relations.
  These mechanisms enable \tool{} to be the first approach to verify standard architectures,
  scaling to official HuggingFace models up to GPT-2 Medium (24 blocks, 300M+ parameters) and successfully verifying 1,339 instances across text and vision tasks.
\end{abstract}


\section{Introduction}
  \label{sec:intro}

  Transformer-based models~\citep{vaswani2017attention}, \eg, generative pretrained transformer (GPT)~\citep{radford2019language},
  are widely used for reasoning, coding, and multimodal agentic tasks~\citep{comanici2025gemini,liu2024deepseek}.
  However, adversarial prompts can change task predictions~\citep{xu2024anllm} or bypass model safeguards~\citep{zou2023universal,chu2025jailbreakradar}, and learning-based systems in other domains are likewise attackable~\citep{le2025fggm}.
  Furthermore, compression and quantization allow open-weight models to run on consumer hardware without cloud infrastructure~\citep{yang2025qwen3}.
  Local deployments can bypass cloud moderation and increase misuse risks~\citep{huang2024position}.

  As traditional software, empirical robustness evaluation is necessary but insufficient since it only exposes failures on tested inputs~\citep{zhang2025evaluating,joo2025harmful}.
  To provide formal assurance that these models have desirable behaviors, such as robustness~\citep{jia2019certified,zhu2024promptbench},
  safety~\citep{zou2023universal} and fairness~\citep{gallegos2024bias,wang2023decodingtrust},
  researchers have developed Neural Network Verification (NNV) techniques that prove required properties of neural networks,
  which include Transformer-based models, and provide auditable guarantees before release or deployment~\citep{shi2020robustness,bonaert2021fast,huang2026parameterized,duong2025generating}.

  While abstraction~\citep{cousot1977abstract,singh2018fast,wang2021beta} has successfully scaled NNV to large models~\citep{kaulen20256thinternationalverificationneural},
  verification remains limited to small or restricted transformers.
  Existing work handles at most three small simplified blocks~\citep{shi2020robustness,huang2026parameterized}.
  \citet{bonaert2021fast} scales to 12 blocks with much smaller configurations.
  None of these tools scale to standard transformer architectures, \eg, GPT-2~\citep{radford2019language}.
  We discuss these works in more detail in \autoref{sec:related}.


  In addition, beyond scalability, maintaining precision across deep transformers remains challenging, \eg, 24 blocks in GPT-2 Medium.
  Abstract domains inherently lose precision as they propagate through successive non-linear layers~\citep{wang2018formal,singh2019abstract},
  \eg, while \covenn{} scales via decomposition~\citep{duong2025compositional}, its interval bounds between blocks substantially degrade precision.
  Furthermore, recent techniques only tighten isolated operations, such as softmax~\citep{wei2023convex} or attention products~\citep{zhang2024galileo,huang2026parameterized},
  and does not address the over-approximation error that accumulates across layers.

  In this paper, we introduce \tool{}, an abstract domain for verifying large transformers that maintains a space complexity
  (\eg, memory footprint) independent of network depth while preserving relations between tokens and internal features.
  First, \tool{} uses structured zonotopes capturing:
  (i) global sequence correlations through shared generators
  and (ii) localized variations through local generators (\autoref{sec:method-domain}).
  Second, reduction mechanisms limit number of generators after each block (\autoref{sec:method-precision}),
  ensuring memory footprint remain bounded and independent of network depth (\autoref{sec:method-complexity}).

  To maintain precision, \tool{} introduces block-specific abstract transformations tailored to transformer components.
  Particularly, \tool{} applies a fused transform for attention residual block (\autoref{sec:fused-attn}) that jointly over-approximates the skip and attention branches.
  By preserving the shared generator structure between the two, \tool{} yields a tighter bound than bounding each branch independently.
  For the MLP residual block,
  \tool{} further improves precision by employing a fused transform for LayerNorm
  to preserve feature relations within each token (\autoref{sec:layernorm-fused}),
  along with an affine transform for GELU (\autoref{sec:gelu-bound}).

  Evaluated on 2,880 text and vision instances,
  \tool{} successfully verifies 1,339 instances (46.5\%),
  scaling to official HuggingFace models up to GPT-2 Medium (24 blocks, 300M+ parameters).
  As perturbation increases, fused attention proves critical for deep configurations (4+ blocks),
  \eg, 867 verified instances (40.1\%), while unfused variants verify fewer than 1\%.

  \noindent\textbf{Contributions.} The primary contributions of this paper include:
  (i) A structured zonotope (\autoref{sec:method-domain}) and generator reduction mechanism (\autoref{sec:method-precision})
  that preserve global and local relations with a depth-independent space complexity;
  (ii) Fused transforms for attention (\autoref{sec:fused-attn}) and LayerNorm (\autoref{sec:layernorm-fused}) that preserve feature relations,
  and an affine GELU transform (\autoref{sec:gelu-bound}) that retains input generators; 
  (iii) Verification of official HuggingFace GPT-2 architectures adapted for vision tasks, scaling to GPT-2 Medium (300M+ parameters) under full-dimensional input perturbations (\autoref{sec:evaluation}).

  While the GPT models considered by \tool{} is still very small compared to the largest ones,
  it is the first approach to verify standard architectures,
  scaling to official HuggingFace models up to GPT-2 Medium (24 blocks, 300M+ parameters) (\autoref{sec:evaluation}).
  We view this as a first step towards verifying larger GPT models,
  and we hope that our work will inspire future research in this direction.

\section{Background}
  \label{sec:background}

  \noindent\textbf{Neural Network Verification.}
    Given a neural network \(N\) and a property $\phi$, the NNV problem asks whether $\phi$ is a valid property of $N$.
    Typically, $\phi$ is a formula $\phi_{in} \Rightarrow \phi_{out}$, where $\phi_{in}$ and $\phi_{out}$ are properties over the inputs and outputs of $N$, respectively.
    An NNV tool searches for a \emph{counterexample} input to $N$ that satisfies $\phi_{in}$ but violates $\phi_{out}$: if none exists,
    $\phi$ is valid; otherwise, $\phi$ is invalid~\citep{wang2021beta,ferrari2022complete,wu2024marabou,duong2023dpll,duong2024harnessing,duong2025neuralsat,duong2025neuralsat2}.
    Modern verifiers often use branch-and-bound~\citep{bunel2020branch,duong2026verifying3},
    whose bounding step computes all reachable outputs efficiently~\citep{singh2018fast,singh2019abstract,wang2018formal,zhou2024scalable,zhou2025clipandverify}.

    Among those, zonotope~\citep{singh2018fast} is the abstraction that balances precision and scalability:
    \begin{equation}
    \mathcal Z(c,G)=\{c+G\epsilon: \epsilon\in[-1,1]^m\}
    \label{eq:background-zonotope}
    \end{equation}
    where $c$ is the center and each column of $G$ is a generator controlled by one coefficient in $\epsilon$.
    The same coefficient can influence multiple outputs, allowing a zonotope to preserve their relations.
    Zonotope transformation for affine layer
    is exact, \eg, $W\mathcal Z(c,G)+b=\mathcal Z(Wc+b,WG)$, while nonlinear operations require additional generators to bound their approximation errors.
    These generators improve precision but make the computational complexity grow with network sizes.

    \tool{} builds on zonotopes with shared generators for cross-token and local generators for token-specific relations (\autoref{sec:method-domain}).
    Generator reduction then limits their number after each block (\autoref{sec:method-precision}), preserving the relations needed by Transformer operations without growth across network depth.

  \noindent\textbf{Residual Taylor Expansion.}
    Let $A(x)=x+B(x)$ be a twice continuously differentiable residual, $Z=\{c+\sum_jH_j\epsilon_j:|\epsilon_j|\le1\}$ be its input zonotope.
    For every $x\in Z$, the change from the center is $v=x-c=\sum_jH_j\epsilon_j$.
    Define $g_v(t)=A(c+tv)$ for $t\in[0,1]$.
    Taylor's theorem along $g_v$ gives
    \begin{equation}
      A(c+v)=A(c)+J_A(c)v+e_A(c,v), \qquad e_A(c,v)=\int_0^1(1-t)g_v''(t)\,dt
      \label{eq:method-first-order-taylor}
    \end{equation}
    where $J_f(c)$ denotes the Jacobian matrix of a function $f$ evaluated at $c$.
    Since $A(x)=x+B(x)$, the identity $J_A(c)=I+J_B(c)$ turns the linear term in \autoref{eq:method-first-order-taylor} into $\sum_j(H_j+J_B(c)H_j)\epsilon_j$.
    Thus, the Taylor expansion of $A(Z)$ is:
    \begin{equation}
      A(Z)\subseteq Z'=\Big\{
      A(c)
      +\sum_j\bigl(H_j+J_B(c)H_j\bigr)\epsilon_j
      +e_A(c,v):
      |\epsilon_j|\le1
      \Big\}
      \label{eq:method-fused-construction}
    \end{equation}
    \tool{} uses \autoref{eq:method-fused-construction} to construct the output zonotope for attention (\autoref{sec:fused-attn}) and LayerNorm (\autoref{sec:layernorm-fused}).

\section{Motivating Example}
  \label{sec:example}

  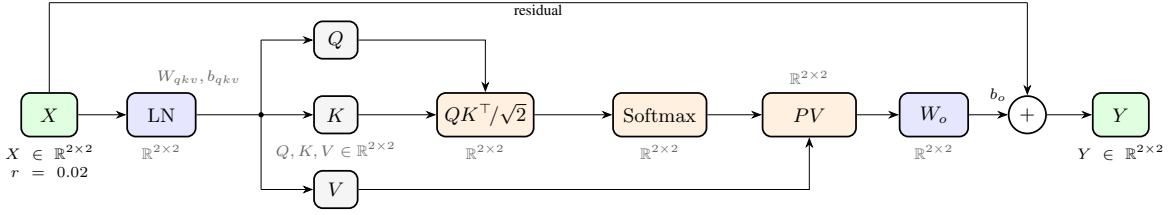
\begin{figure}[t]
    \centering
    \resizebox{\linewidth}{!}{%
    \begin{tikzpicture}[
     >=Stealth,
     every node/.style={font=\small},
     tensor/.style={rectangle, rounded corners, draw, thick,
       minimum height=7mm, minimum width=9mm, fill=green!12},
     qkv/.style={rectangle, rounded corners, draw, thick,
       minimum height=6mm, minimum width=7mm, fill=black!4},
     op/.style={rectangle, rounded corners, draw, thick,
       minimum height=7mm, minimum width=11mm, fill=blue!10},
     attn/.style={rectangle, rounded corners, draw, thick,
       minimum height=7mm, minimum width=15mm, fill=orange!12},
     sumnode/.style={circle, draw, thick, minimum size=6mm, inner sep=0pt},
     elbl/.style={font=\scriptsize, above=1mm},
     spec/.style={font=\scriptsize, text width=26mm, align=center},
     shp/.style={font=\scriptsize, text=black!60},
    ]
    \node[tensor,
          label={[spec]below:$X\in\R^{2\times2}$ \\ $r=0.02$}]
          (x)    at (0,0)      {$X$};
    \node[op, label={[shp]below:$\R^{2\times2}$}]
          (ln)   at (1.8,0)    {$\mathrm{LN}$};
    \node[qkv]    (q)    at (4.6,1.2)  {$Q$};
    \node[qkv, label={[shp]below:$Q,K,V\in\R^{2\times2}$}]
          (k)    at (4.6,0)    {$K$};
    \node[qkv]    (v)    at (4.6,-1.2) {$V$};
    \node[attn, label={[shp]below:$\R^{2\times2}$}]
          (qk)   at (7.0,0)    {$QK^\top\!/\sqrt2$};
    \node[attn, label={[shp]below:$\R^{2\times2}$}]
          (sm)   at (9.8,0)    {$\mathrm{Softmax}$};
    \node[attn, label={[shp]above:$\R^{2\times2}$}]
          (pv)   at (12.2,0)   {$PV$};
    \node[op, label={[shp]below:$\R^{2\times2}$}]
          (wo)   at (14.2,0)   {$W_o$};
    \node[sumnode](sum)  at (15.7,0)   {$+$};
    \node[tensor,
          label={[spec]below:$Y\in\R^{2\times2}$}]
          (y)    at (17.2,0)   {$Y$};

    \draw[->] (x) -- (ln);
    \draw[->] (ln) -- (k);
    \draw[->] (qk) -- (sm);
    \draw[->] (sm) -- (pv);
    \draw[->] (pv) -- (wo);
    \draw[->] (wo) -- (sum) node[elbl, pos=0.75] {$b_o$};
    \draw[->] (sum) -- (y);

    \node[shp] at (2.4,0.6) {$W_{qkv},b_{qkv}$};

    \draw (3.4,1.2) -- (3.4,-1.2);
    \fill (3.4,0) circle (1.2pt);
    \draw[->] (3.4,1.2) -- (q);
    \draw[->] (3.4,-1.2) -- (v);

    \draw[->] (q) -| (qk);
    \draw[->] (k) -- (qk);
    \draw[->] (v) -| (pv);

    \draw (x.north) -- (0,1.8);
    \draw (0,1.8) -- (15.7,1.8) node[elbl, pos=0.5, below] {residual};
    \draw[->] (15.7,1.8) -- (sum.north);
    \end{tikzpicture}%
    }
    \caption{Running example for the attention residual for $S=2$ tokens, hidden size $d=2$, one head.}
    \label{fig:toy-transformer}
  \end{figure}

  We demonstrate \tool{} on an attention-residual block with $S=2$ tokens and hidden size $d=2$ (\autoref{fig:toy-transformer}).
  We verify that its first output coordinate remains positive given an $\ell_\infty$ perturbation:
  \begin{equation}
    \begin{aligned}
      \varphi&\equiv\|x-x_0\|_\infty\le r\Longrightarrow[A(x)]_{1,1}>0, \qquad
      x_0=[[0.0, 0.0], [0.7, 0.8]],\qquad r=0.02
    \end{aligned}
    \label{eq:ex-gpt-property}
  \end{equation}

  \textbf{Scalability.}
  \deepz{} represents the input with $Sd$ dense generators, each containing $Sd$ coefficients, and therefore stores $S^2d^2$ coefficients at the input.
  For the running example ($S=2$, $d=2$):
  \begin{equation}
   \begin{aligned}
   x_{11} &= x_{11}^0+r\epsilon_1+0\epsilon_2+0\epsilon_3+0\epsilon_4 &
   x_{12} &= x_{12}^0+0\epsilon_1+r\epsilon_2+0\epsilon_3+0\epsilon_4 \\
   x_{21} &= x_{21}^0+0\epsilon_1+0\epsilon_2+r\epsilon_3+0\epsilon_4 &
   x_{22} &= x_{22}^0+0\epsilon_1+0\epsilon_2+0\epsilon_3+r\epsilon_4
   \end{aligned}
   \label{eq:ex-deepz-input}
  \end{equation}
  \deepz{} stores all 16 coefficients in \autoref{eq:ex-deepz-input}, \eg, $(r,0,0,0)$ for $x_{11}$, although each generator affects only one token.
  \tool{} stores the same input without the cross-token zeros:
  \begin{equation}
   x_{11} = x_{11}^0+r\epsilon_1+0\epsilon_2 \qquad
   x_{12} = x_{12}^0+0\epsilon_1+r\epsilon_2 \qquad
   x_{21} = x_{21}^0+r\epsilon_3+0\epsilon_4 \qquad
   x_{22} = x_{22}^0+0\epsilon_3+r\epsilon_4
   \label{eq:ex-local-input}
  \end{equation}
  \tool{} stores eight coefficients (instead of the 16).
  For the GPT-2-small input with $S=512$ and $d=768$, \tool{} uses $Sd^2\cdot8/2^{30}=2.25$ GiB, while \deepz{} uses $S^2d^2\cdot8/2^{30}=1152$ GiB.
  For hidden states after each residual component, \tool{} stores generators at the source token:
  \begin{equation}
   Z_s=c_s+\sum_{j=1}^{m}G_{j,s}\epsilon_j+L_s\epsilon_s^{\mathrm{loc}}+[-b_s,b_s]
   \label{eq:ex-structured-zono}
  \end{equation}
  \tool{} uses $G$ for shared generators, $L_s$ for generators local to token $s$, and $b_s$ for the remaining independent values.
  With at most $m$ shared and $q$ local generators per token, these terms require $O(Sd(m+q+1))$ coefficients independent of network depth (\autoref{sec:method-complexity}).
  In contrast, \deepz{} adds $O(Sd)$ dense generators per block.
  Each dense generator contains $Sd$ coefficients, so one block adds $O(S^2d^2)$ coefficients and $L$ blocks add $O(LS^2d^2)$.

  \begin{figure}[t]
    \centering
    \includegraphics[width=\linewidth]{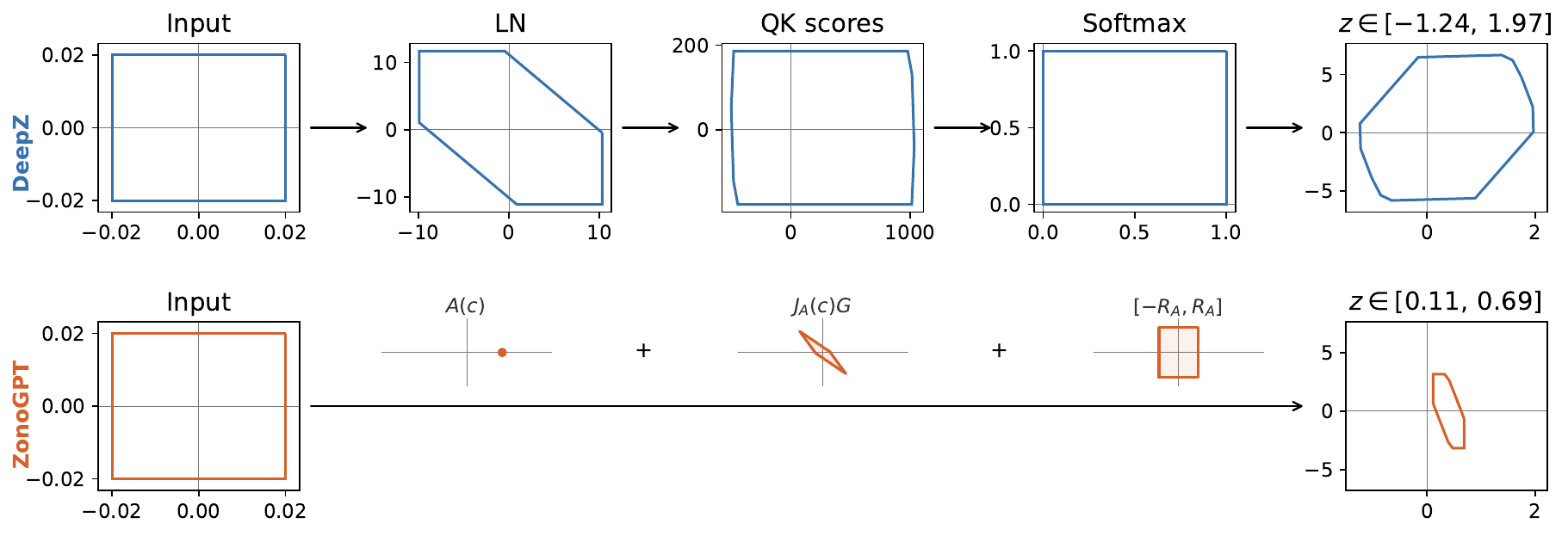}
    \caption{First-token projections for operator-wise \deepz{} (top) and fused \tool{} (bottom).}
    \label{fig:running-sets}
  \end{figure}

  \textbf{Precision.}
  We use the network in \autoref{fig:toy-transformer} with the following fixed parameters:
  \begin{equation}
   W_{qkv}^{\top}=\begin{bmatrix}
   3.1&-0.2&1.0&0.1&0.8&-3.0\\
   0.7&-1.1&-4.3&1.2&-0.5&0.6
   \end{bmatrix},\qquad
   W_o=\begin{bmatrix}0.112&0.029\\0.014&0.153\end{bmatrix}
   \label{eq:ex-weights}
  \end{equation}
  $b_{qkv}=(-0.1,0.1,0,0.2,0.2,-0.2)$, $b_o=(0.32122,0.95628)$, $\gamma=(1.6,1.8)$, $\beta=(0.2,0.3)$, and $\tau=10^{-5}$.
  We represent the input interval by the zonotope
  \begin{equation}
    X=\begin{bmatrix}0&0\\0.7&-0.8\end{bmatrix}
      +0.02\begin{bmatrix}\epsilon_{11}&\epsilon_{12}\\
                          \epsilon_{21}&\epsilon_{22}\end{bmatrix},\qquad
    c=\begin{bmatrix}0&0\\0.7&-0.8\end{bmatrix},\qquad
    \epsilon_{ij}\in[-1,1]
    \label{eq:ex-precision-input}
  \end{equation}
  The top row of \autoref{fig:running-sets} follows \deepz{} through LayerNorm, $QK^\top$, and softmax.
  Each nonlinear operator bounds its input with fresh generators and the attention output has growth to 34 generators:
  \begin{equation}
    [A(X)]_{1,1}=z_0+\textstyle\sum_{i=1}^{34}g_i\epsilon_i,\qquad
   z_0\approx0.36,\qquad \textstyle\sum_{i}|g_i|\approx1.61
   \label{eq:ex-deepz-zono}
  \end{equation}
  Concretizing \autoref{eq:ex-deepz-zono} gives $[A(X)]_{1,1}\in[-1.24,1.98]$ (endpoints rounded outward), and cannot verify the property of $[A(X)]_{1,1} > 0$ (\autoref{eq:ex-gpt-property}).

  \tool{} instead bounds the attention-residual component as one \emph{fused} block (\autoref{sec:fused-attn}):
  \begin{equation}
   A(x)\in A(c)+J_A(c)G\epsilon+[-R_A,R_A], \qquad \|x-x_0\|_\infty\le r
   \label{eq:ex-fused-zono}
  \end{equation}
  The bottom row of \autoref{fig:running-sets} separates the three contributions in \autoref{eq:ex-fused-zono}.
  The output center $A(c)$ computes the attention output at the input center $x=c$.
  The Jacobian image $J_A(c)G$ carries each input generator through both skip and attention branch, preserving correlations between the two branches.
  The fused generators encode the nonlinear remainder bounded by $R_A$.
  Adding the three terms together gives the final bound, $[A(X)]_{1,1}\in[0.11,0.69]$, and verifies $[A(X)]_{1,1}>0$.

\section{The \tool{} Approach}
  \label{sec:algorithm}

  \begin{algorithm}[t]
  \caption{Structured-zonotope bound propagation}
  \label{alg:method}
  \KwIn{$F=\Phi_T\circ\cdots\circ \Phi_1$; input center and radius $(x_0,r)$; $m$ shared and $q$ local generators}
  \KwOut{$(\underline y,\overline y)$ bounding $F(x)$ for every $\|x-x_0\|_\infty \le r$}

  $Z\leftarrow \mathcal{Z}(x_0,0,\{\operatorname{diag}(r_s)\}_{s=1}^S,0)$ \tcp*{Initialize local input generators (\autoref{sec:method-domain})}\label{alg:method:init}
  \For{$t\leftarrow1$ \KwTo $T$}{
    \For(\tcp*[f]{Represent $|e_s|\le b_s$ with local generators (\autoref{sec:method-precision})}){$s\leftarrow1$ \KwTo $S$}{
      $I_s\leftarrow\{i:(b_s)_i>0\}$;
      $L_s\leftarrow[L_s\mid\operatorname{diag}(b_s)_{:,I_s}]$;
      $b_{s,I_s}\leftarrow0$\;\label{alg:method:convert}
    }

    \eIf{$\Phi_t$ is an attention block}{
      $Z\leftarrow\mathsf{FusedAttentionTransform}(\Phi_t,Z)$ \tcp*{Bound Attention residual (\autoref{sec:fused-attn})}\label{alg:method:fused}
    }{
      $Z\leftarrow\mathsf{SequentialMLPTransform}(\Phi_t,Z)$ \tcp*{Bound MLP residual (\autoref{sec:unfused-mlp})}\label{alg:method:layer}
    }

    $K\leftarrow\mathsf{SelectShared}(G,m)$ \tcp*{Select at most $m$ shared generators (\autoref{sec:method-precision})} \label{alg:method:select}
    $b\leftarrow b+\sum_{j\notin K}|G_j|;\qquad G\leftarrow G_K$\;\label{alg:method:reduce}
    \For(\tcp*[f]{Select at most $q$ local generators (\autoref{sec:method-precision})}){$s\leftarrow1$ \KwTo $S$}{
      $K_s\leftarrow\mathsf{SelectLocal}(L_s,q);\quad b_s\leftarrow b_s+\sum_{j\notin K_s}|L_{s,j}|; \quad L_s\leftarrow L_{s,K_s}$\;\label{alg:method:reduce-local}
    }
  }
  \Return{$\operatorname{concretize}(Z)$}\;\label{alg:method:hull}
  \end{algorithm}

  \autoref{alg:method} describes how \tool{} computes bounds for an input interval $(x_0,r)$ through a Transformer $F$.
  First, \tool{} represents the input interval with local generators (\autoref{alg:method:init}).
  For each residual component, \tool{} represents $|e|\le b$ with zonotope generators (\autoref{alg:method:convert}).
  It applies fused transformation to the attention block (\autoref{alg:method:fused}) and propagates layer-by-layer through the MLP block (\autoref{alg:method:layer}).
  \tool{} limits shared generators to $m$ (\autoref{alg:method:select} to \autoref{alg:method:reduce}).
  It also limits local generators to $q$ per token (\autoref{alg:method:reduce-local}).
  Finally, \tool{} returns the output bounds (\autoref{alg:method:hull}).

  The key idea is to preserve dependencies that can tighten later bounds.
  During bound propagation, \tool{} uses \emph{shared generators} for cross-token relations (expensive) and \emph{local generators} for token-specific relations (linear cost in the sequence length), while \emph{intervals} keep the number of generators from growing with network depth.

  \subsection{Structured Representation}
  \label{sec:method-domain}

  \tool{} represents the input using local generators (\autoref{alg:method:init}) because each input coordinate initially affects only its token.
  Attention later mixes tokens, so \tool{} uses shared generators to preserve cross-token relations and an interval to bound the remaining terms independently.
  This design is inspired by our recent work on preserving cross-dimension relations in input properties~\citep{duong2026verifying,duong2026verifying2}.

  \begin{definition}[Structured Zonotope]
  \label{def:method-domain}
  For $S$ tokens, let $c$ be the center, $m$ shared generators $G$, $q_s$ local generators $L_s$ at token $s$, and $b\ge0$ the interval radius.
  The represented set is:
  \begin{equation}
   \mathcal Z(c,G,L,b)=\Big\{c+\sum_{j=1}^{m}G_j\epsilon_j+
   \sum_{s=1}^{S}\operatorname{emb}_s(L_s\epsilon^{\mathrm{loc}}_s)+e:
   \|\epsilon\|_\infty\le1,\|\epsilon^{\mathrm{loc}}_s\|_\infty\le1,\ |e|\le b\Big\}
  \end{equation}
  $\operatorname{emb}_s(L_s\epsilon^{\mathrm{loc}}_s)$ forms a matrix whose $s$-th row is $L_s\epsilon^{\mathrm{loc}}_s$ and whose other rows are zero.
  A generator $G_j$ can change several tokens through the same $\epsilon_j$, while a generator in $L_s$ changes only token $s$.
  \end{definition}

  The coordinate radius of $Z=\mathcal Z(c,G,L,b)$ is $r(Z)$ and interval bounds $[c-r(Z),c+r(Z)]$:
  \begin{equation}
   r(Z)=\sum_{j=1}^{m}|G_j|
        +\sum_{s=1}^{S}\operatorname{emb}_s\!\Big(\sum_{j=1}^{q_s}|L_{s,j}|\Big)+b
   \label{eq:method-radius}
  \end{equation}

  \subsection{Fused Attention Block Transform}
  \label{sec:fused-attn}

  For input $x$, let $A(x)$ denote the attention output with the skip connection.
  Let $n$ be the LayerNorm output.
  For head $h$, let $Q_h$, $K_h$, and $V_h$ be the projected queries, keys, and values.
  Let $U_h$ be the score matrix and $P_h$ the softmax output:
  \begin{align}
    \label{eq:method-attention}
   n&=\operatorname{LN}(x), \qquad Q\|K\|V=W_{qkv}n+b_{qkv}, \qquad U_h=Q_hK_h^\top/\sqrt{d_h}, \qquad P_h=\operatorname{softmax}(U_h) \nonumber\\
   A(x)&=x+W_o\underbrace{[\,P_1V_1\mid\cdots\mid P_HV_H\,]}_{\operatorname{concat}(P_hV_h)}+b_o
  \end{align}

  Note that attention contains the related products $QK^\top$ and $PV$, while softmax couples every key in a row of $U_h$.
  Bounding these operations separately loses relations among their terms before adding to the skip branch.
  \tool{} instead propagates each input generator through the complete attention residual and bounds only the remaining nonlinear error.

  To construct output bounds for $A$, \tool{} computes the three terms in \autoref{eq:method-fused-construction}, \eg, $A(c)$, $H_j+J_B(c)H_j$, and $e_A(c,v)$.
  First, \tool{} computes $A(c)$ simply by evaluating \autoref{eq:method-attention} at $x=c$.

  Second, \tool{} computes how $A$ changes when the coefficient of $H_j$ varies from zero:
  \begin{equation}
   x(t)=c+tH_j, \qquad
   \dot x=\left.\frac{dx(t)}{dt}\right|_{t=0}=H_j, \qquad
   \dot A=\left.\frac{dA(x(t))}{dt}\right|_{t=0}=J_A(c)H_j
   \label{eq:attention-direction}
  \end{equation}
  Applying the chain rule along this path gives:
  \begin{equation}
    \label{eq:method-attention-jvp}
    \begin{aligned}
      \dot n&=J_{\rm LN}(c)\dot x, \qquad
      \dot Q\|\dot K\|\dot V=W_{qkv}\dot n, \qquad
      \dot U=(\dot QK^\top+Q\dot K^\top)/\sqrt{d_h}, \qquad
      \dot P=J_{\rm SM}(U)\dot U\\
      \dot A&=\dot x+W_o\operatorname{concat}(\dot P_hV_h+P_h\dot V_h)
    \end{aligned}
  \end{equation}
  where each undotted quantity
  (\eg, $Q$, $K$, $V$, $U$, and $P$) denotes its value evaluated at $x=c$\tl{, and $J_{\rm SM}(U)$ is the Jacobian of softmax at the scores $U$}.
  \autoref{eq:attention-direction} and \autoref{eq:method-attention-jvp} give $J_A(c)H_j = H_j+J_B(c)H_j = \dot x+W_o\operatorname{concat}(\dot P_hV_h+P_h\dot V_h)$.

  Third, Taylor expansion identifies error term $e_A(c,v)$ as $A(c+v)-A(c)-J_A(c)v$ (\autoref{eq:method-first-order-taylor}).
  \tool{} then computes $R_A$ satisfying:
  \begin{equation}
   |e_A(c,v)| = |A(c+v)-A(c)-J_A(c)v|\le R_A
   \label{eq:method-fused-error-contract}
  \end{equation}
  Intuitively, $R_A$ bounds the nonlinear terms introduced by $QK^\top$, softmax, and $PV$.
  \tool{} computes a bound $R_A$ for the nonlinear attention error $e_A$ using~\autoref{thm:attention-radius}.
  \begin{theorem}[Attention nonlinear-error bound]
  \label{thm:attention-radius}
  Let $A$ be the attention residual in \autoref{eq:method-attention}, and the input zonotope be $Z=\{c+\sum_{j=1}^mH_j\epsilon_j:|\epsilon_j|\le1\}$.
  Let $P=P(c)$ and $V=V(c)$.
  Assume nonnegative tuples
  $(f_P, R_P, f_V, R_V)$ denote the bounds on the first-order change ($f$) and the Taylor remainder ($R$) of $P$ and $V$, respectively.
  satisfy $|J_P(c)(x-c)|\le f_P$, $|P(x)-P-J_P(c)(x-c)|\le R_P$, $|J_V(c)(x-c)|\le f_V$, and $|V(x)-V-J_V(c)(x-c)|\le R_V$ for every $x\in Z$, then
  \begin{equation}
   \underbrace{|A(x)-A(c)-J_A(c)(x-c)|}_{e_A}
   \le
   \underbrace{|W_o|\operatorname{concat}\left(
   R_P|V|+PR_V+(f_P+R_P)(f_V+R_V)
   \right)}_{R_A}
   \label{eq:attention-radius}
  \end{equation}
  \end{theorem}
  \begin{proof}[Proof sketch]
  Expanding the attention-value product gives the three terms on the right-hand side of \autoref{eq:attention-radius}.
  The triangle inequality bounds these terms by $R_P|V|$, $PR_V$, and $(f_P+R_P)(f_V+R_V)$, respectively.
  \tool{} derives $f_P,R_P,f_V,$ and $R_V$ in \autoref{app:attention-remainder}.
  \end{proof}

  \autoref{thm:attention-radius} bounds the nonlinear attention error with an interval.
  \tool{} improves precision by representing the part caused by the softmax error (\eg, $|W_o|\operatorname{concat}(R_P|V|)$) with a zonotope.
  Applying the residual Taylor expansion in~\autoref{eq:method-first-order-taylor} to $P$ gives $P(x)=P(c)+J_P(c)(x-c)+e_P$, hence $e_P=P(x)-P(c)-J_P(c)(x-c)$.
  The entry $(e_P)_{ihj}$ is the softmax error for query token $i$, attention head $h$, and key token $j$.
  Each row of $P(x)$ and $P(c)$ sums to one (softmax normalizes each row), and each row of $J_P(c)(x-c)$ sums to zero (the derivative of this fixed row sum is zero).
  Thus, $\sum_j (e_P)_{ihj}=1-1-0=0$.

  Let $\bar v_{hj}=W_o^{(h)}V_{hj}(c)$ be the projected center value at key token $j$.
  Because the coefficients $(e_P)_{ihj}$ sum to zero, subtracting $a_{ih}$ from every $\bar v_{hj}$ does not change their weighted sum.
  Since $|e_P|\le R_P$, we write $(e_P)_{ihj}=\xi_{ihj}(R_P)_{ihj}$ with $|\xi_{ihj}|\le1$, which gives
  \begin{equation}
   \begin{aligned}
   \sum_j (e_P)_{ihj}\bar v_{hj}
   &=\sum_j (e_P)_{ihj}(\bar v_{hj}-a_{ih})\\
   &=\sum_j \xi_{ihj}(R_P)_{ihj}(\bar v_{hj}-a_{ih})
   \end{aligned}
   \label{eq:attention-zero-sum}
  \end{equation}
  Thus, the softmax term is a combination of the generators $(R_P)_{ihj}(\bar v_{hj}-a_{ih})$ with coefficients $\xi_{ihj}$, and their radius $\sum_j (R_P)_{ihj}|\bar v_{hj}-a_{ih}|$ depends on $a_{ih}$.
  \tool{} chooses $a_{ih}$ as an elementwise weighted median of $\bar v_{hj}$ with weights $(R_P)_{ihj}$, which minimizes this radius in output. 

  Finally, \tool{} computes a structured zonotope $Z_A=\mathcal Z(c_A,G_A,L_A,b_A)$ (\autoref{def:method-domain}) for $A$:
  \begin{align}
    \label{eq:attn-zonotope-coefficients}
    c_A&=A(c), \qquad (G_A)_j=J_A(c)G_j, \nonumber\\
   (L_A)_i&=[\,(J_A(c)\operatorname{emb}_i(L_{i,k}))_i\,]_{k=1}^{q_i}
      \;\big|\;[\,(R_P)_{ihj}(\bar v_{hj}-a_{ih})\,]_{h,j}, \nonumber\\
   (b_A)_i&=\left[|W_o|\operatorname{concat}\left(PR_V+(f_P+R_P)(f_V+R_V)\right)\right]_i
      +\sum_{s\ne i}\sum_{k=1}^{q_s}|(J_A(c)\operatorname{emb}_s(L_{s,k}))_i|
   \end{align}
  \autoref{thm:contracted-attention} establishes soundness of the output zonotope in \autoref{eq:attn-zonotope-coefficients}.
  \begin{theorem}[Fused attention soundness]
  \label{thm:contracted-attention}
  Let the input zonotope be $Z=\mathcal Z(c,G,L,0)$ with local generators $L_i=[L_{i,1},\ldots,L_{i,q_i}]$, and let $A$ be the attention in \autoref{eq:method-attention}.
  Suppose a nonnegative tuple $(f_P,R_P,f_V,R_V)$ satisfies the four bounds in \autoref{thm:attention-radius} for every $x\in Z$.
  For each query token $i$ and head $h$, choose $a_{ih}$ shared across all key tokens $j$.
  Then $Z_A=\mathcal Z(c_A,G_A,L_A,b_A)$ defined by \autoref{eq:attn-zonotope-coefficients} satisfies $A(Z)\subseteq Z_A$.
  \end{theorem}
  \begin{proof}[Proof sketch]
  \autoref{eq:attention-zero-sum} represents the softmax error with the local generators in $L_A$.
  The other nonlinear terms satisfy the bounds in \autoref{thm:attention-radius}, while $b_A$ bounds the cross-token parts of the local input generator images.
  The full proof is in \autoref{app:fused-attn-sound-proof}.
  \end{proof}

  \subsection{MLP Block Transform}
  \label{sec:unfused-mlp}

  For input $x$, the MLP block computes
  \begin{equation}
   M(x)=x+W_2g(W_1\operatorname{LN}(x)+b_1)+b_2
   \label{eq:method-unfused-mlp}
  \end{equation}
  The MLP contains two nonlinear operations with different structures.
  LayerNorm couples the features within each token through their mean and variance, so \tool{} bounds the complete LayerNorm as one fused transform (\autoref{sec:layernorm-fused}).
  GELU acts independently on each coordinate, and the surrounding affine maps are exact (\autoref{sec:gelu-bound}).
  After LayerNorm, the MLP has no bilinear products such as $QK^\top$ and $PV$ in attention (\autoref{sec:fused-attn}).
  \tool{} therefore propagates the affine maps and GELU sequentially before combining the block with the skip connection.

  \subsubsection{LayerNorm}\label{sec:layernorm-fused}
  The mean and variance in LayerNorm depend on the same token features.
  Bounding these quantities separately loses this relation, so \tool{} treats the complete LayerNorm as one fused transform using the Taylor expansion in \autoref{eq:method-fused-construction}.
  At token $s$, the input zonotope has the form:
  \begin{equation}
   x_s=c_s+\sum_j(G_j)_s\epsilon_j+L_s\epsilon_s^{\rm loc}
   \label{eq:method-ln-input}
  \end{equation}
  Applying the Taylor expansion to the complete LayerNorm gives
  \begin{equation}
   \operatorname{LN}(x_s)
   =\operatorname{LN}(c_s)
     +\sum_jJ_{\rm LN}(c_s)(G_j)_s\epsilon_j
     +J_{\rm LN}(c_s)L_s\epsilon_s^{\rm loc}
     +e_{N,s},\qquad
   |e_{N,s}| \le (R_N)_s
  \label{eq:method-ln-expansion}
  \end{equation}
  Thus, \tool{} computes the LayerNorm output in the form of a structured zonotope as:
  \begin{equation}
   Z_N=\mathcal Z\!\left(
   \operatorname{LN}(c),
   \{J_{\rm LN}(c)G_j\}_j,
   \{J_{\rm LN}(c_s)L_s\}_{s=1}^{S},
   R_N\right)
   \label{eq:method-ln-zonotope}
  \end{equation}
  where $c=(c_1,\ldots,c_S)$ is the center, and $\operatorname{LN}(c)=(\operatorname{LN}(c_1),\ldots,\operatorname{LN}(c_S))$.
  \tool{} derives the remaining term $R_N$ and proves its soundness (\autoref{thm:layernorm-soundness}) in \autoref{app:layernorm}.

  \subsubsection{GELU}\label{sec:gelu-bound}
  The MLP applies tanh-GELU elementwise:
  \begin{equation}
   g(t)=\frac{t}{2}\Big[1+\tanh\!\Big(\sqrt{\frac{2}{\pi}}\left(t+0.044715t^3\right)\Big)\Big]
   \label{eq:method-gelu-definition}
  \end{equation}
  For a GELU input zonotope $Z=\mathcal Z(c,G,L,b)$, \tool{} first computes the coordinate interval $[\ell,u]$.
  For each coordinate with $\ell_k<u_k$, \tool{} separates GELU into the linear term $a_kt$ and the remaining error $g(t)-a_kt$.
  \begin{equation}
   a_k=\frac{g(u_k)-g(\ell_k)}{u_k-\ell_k},\qquad
   r_k^{\rm lb}\le g(t)-a_kt\le r_k^{\rm ub},\qquad t\in[\ell_k,u_k]
   \label{eq:method-gelu-residual}
  \end{equation}
  where the error bounds $r^{\rm lb}$ and $r^{\rm ub}$ are computed in \autoref{app:gelu-affine}.
  With $m=(r^{\rm lb}+r^{\rm ub})/2$ and $d=(r^{\rm ub}-r^{\rm lb})/2$, \autoref{eq:method-gelu-residual} gives the affine GELU bound:
  \begin{equation}
   g(z)\in a\odot z+m+\{d\odot\epsilon^g:|\epsilon^g|\le1\}
   \label{eq:method-gelu-affine}
  \end{equation}
  Applying \autoref{eq:method-gelu-affine} to $Z$ gives the GELU output:
  \begin{equation}
   Z_g=\mathcal Z\!\left(
   a\odot c+m,
   \{a\odot G_j\}_j,
   \{[\,\operatorname{diag}(a_s)L_s\mid\operatorname{diag}(d_s)\,]\}_{s=1}^{S},
   |a|\odot b
   \right)
   \label{eq:method-gelu-zonotope}
  \end{equation}
  GELU bounds are proved sound in \autoref{thm:gelu-soundness} (\autoref{app:gelu-affine}).

  \subsection{Zonotope Conversion and Reduction}
  \label{sec:method-precision}

  In \autoref{def:method-domain}, $|e|\le b$ term means that every coordinate of $e$ varies independently.
  Before applying block transform, \tool{} writes these values as local zonotope generators (\autoref{alg:method:convert}).
  The skip and branch then use the same $\epsilon$ for each value while maintaining the original set $Z$ (\autoref{lem:method-conversion}).

  \begin{lemma}[Interval-to-zonotope conversion]
  \label{lem:method-conversion}
  A structured zonotope with $S$ tokens satisfies
  \begin{equation}
   \mathcal Z(c,G,L,b)
   =\mathcal Z\!\left(c,G,
   \left\{[L_s\mid\operatorname{diag}(b_s)]\right\}_{s=1}^{S},0\right)
   \label{eq:method-conversion}
  \end{equation}
  \end{lemma}
  \begin{proof}[Proof]
  The condition $|e_s|\le b_s$ is equivalent to $e_s=\operatorname{diag}(b_s)\epsilon_s$ for some $|\epsilon_s|\le1$.
  \end{proof}
  After each block transform, \tool{} retains at most $m$ shared and $q$ local generators per token (\autoref{alg:method:select} to \autoref{alg:method:reduce-local}).
  \tool{} ranks the generators by summed absolute coefficient magnitude.
  The selected shared generators remain in $G$, while $|G_j|$ is added to $b$ for every discarded generator $G_j$.
  The resulting zonotope contains the original zonotope (\autoref{lem:method-reduction}).
  \begin{lemma}[Generator reduction]
  \label{lem:method-reduction}
  For any set $K$ of retained shared-generator indices,
  \begin{equation}
   \mathcal Z(c,G,L,b)
   \subseteq
   \mathcal Z\!\Big(c,G_K,L,b+\sum_{j\notin K}|G_j|\Big)
   \label{eq:method-reduction}
  \end{equation}
  \end{lemma}
  \begin{proof}[Proof]
  For discarded index $j$, $|G_j\epsilon_j|\le|G_j|$ because $|\epsilon_j|\le1$.
  Summing these gives \autoref{eq:method-reduction}.
  \end{proof}

  \subsection{Complexity}
  \label{sec:method-complexity}

  \noindent\textbf{Space complexity.}
  Between components, \autoref{alg:method:select} to \autoref{alg:method:reduce-local} retain at most $m$ shared and $q$ local generators per token for $S$ tokens of width $d$.
  The center, $b$, shared and local generators thus require
  $O(Sd)+O(mSd)+O(Sqd)=O(Sd(m+q+1))$.
  For an attention, \autoref{eq:method-conversion} adds $O(Sd^2)$ temporary coefficients before reduction, while the softmax for $H$ heads adds $O(HS^2d)$ and dominates the $O(HS^2)$ attention matrices.
  The space complexity of \tool{} is therefore:
  \begin{equation}
   O\bigl(Sd(m+q+1)+Sd^2+HS^2d\bigr)
   \label{eq:method-peak-space}
  \end{equation}
  \noindent\textbf{Time complexity.}
  Let $T_\Phi$ be the time for one full-component
  Jacobian vector product (JVP)
  , $T_{\rm tok}$ the time for one source-token attention JVP, and $T_R$ the time for computing the nonlinear remainder bound.
  \tool{} has the following time complexity: 
  \begin{equation}
   \begin{aligned}
   T_{\rm A}=O\bigl(mT_\Phi+S(q+d)T_{\rm tok}+T_R\bigr),\qquad \qquad
   T_{\rm MLP}=O\bigl((m+q+d)T_\Phi+T_R\bigr)
   \end{aligned}
   \label{eq:method-time}
  \end{equation}

\section{Evaluation}
  \label{sec:evaluation}

  \subsection{Experimental Design}
    \begin{table}[t]
      \centering
      \caption{Benchmark models and instances.}
      \vspace{-0.1em}
      \label{tab:eval-models}
      \begin{adjustbox}{width=0.7\linewidth}
      \begin{tabular}{llccccc}
      \toprule
      \textbf{Type} & \textbf{Model} & \textbf{Blocks} & \textbf{Width} & \textbf{Heads} & \textbf{Parameters} & \textbf{Instances} \\
      \midrule
      \multirow{2}{*}{Vision} & GPT-2 Small & $\{2, 4, 8, 12\}$ & 768  & 12 & 14.2M - 85.1M & 720 \\
       & GPT-2 Medium & $\{2, 4, 8, 24\}$ & 1024 & 16 & 25.3M - 302.4M & 720 \\
      \midrule
      \multirow{2}{*}{Text} & GPT-2 Small & $\{2, 4, 8, 12\}$ & 768  & 12 & 52.8M - 123.7M & 720 \\
       & GPT-2 Medium & $\{2, 4, 8, 24\}$ & 1024 & 16 & 76.7M - 353.8M & 720 \\
      \midrule
      \multicolumn{2}{l}{\textbf{Total}} & & & & & \textbf{2880} \\
      \bottomrule
      \end{tabular}
      \end{adjustbox}
    \end{table}

    \begin{table}[t]
      \centering
      \setlength{\tabcolsep}{2pt}
      \small
      \caption{Verification results (runtime in seconds). A dash (-) denotes out-of-memory errors.}
      \vspace{-0.1em}
      \label{tab:verification-results}
      \resizebox{\textwidth}{!}{%
      \begin{tabular}{llrrrrrrrrrrrrrrrrr}
      \toprule
      \multirow{2}{*}{\textbf{Dataset}} & \multirow{2}{*}{\textbf{Model}} & \multirow{2}{*}{\textbf{Block}} & \multicolumn{2}{c}{\textsc{IBP}} & \multicolumn{2}{c}{\deepz{}} & \multicolumn{2}{c}{\textsc{$\alpha\beta$-CROWN}} & \multicolumn{2}{c}{\textsc{CoVeNN}} & \multicolumn{2}{c}{\toolM{}} & \multicolumn{2}{c}{\toolD{}} & \multicolumn{2}{c}{\toolF{}} & \multicolumn{2}{c}{\toolA{}} \\
       &  &  & \textbf{Solved} & \textbf{Time} & \textbf{Solved} & \textbf{Time} & \textbf{Solved} & \textbf{Time} & \textbf{Solved} & \textbf{Time} & \textbf{Solved} & \textbf{Time} & \textbf{Solved} & \textbf{Time} & \textbf{Solved} & \textbf{Time} & \textbf{Solved} & \textbf{Time} \\
      \midrule
      \multirow{8.5}{*}{MNIST} & \multirow{4}{*}{\shortstack{GPT-2\\Small}} & 2 & 0 & 3.6 & - & - & - & - & 0 & 5769.8 & 36 & 215.3 & 36 & 826.2 & 108 & 544.5 & \textbf{144} & 1507.2 \\
       &  & 4 & 0 & 6.5 & - & - & - & - & 0 & 11256.8 & 0 & 464.1 & 0 & 1423.8 & \textbf{108} & 2463.5 & \textbf{108} & 4327.8 \\
       &  & 8 & 0 & 12.2 & - & - & - & - & 0 & 22244.4 & 0 & 824.2 & 0 & 2942.5 & 63 & 5609.2 & \textbf{90} & 10621.9 \\
       &  & 12 & 0 & 18.3 & - & - & - & - & 0 & 33440.1 & 0 & 1232.9 & 0 & 4435.2 & 54 & 6688.8 & \textbf{72} & 15862.0 \\
      \cmidrule{2-19}
       & \multirow{4}{*}{\shortstack{GPT-2\\Medium}} & 2 & 0 & 3.6 & - & - & - & - & 0 & 10472.1 & 36 & 291.9 & 36 & 729.3 & 108 & 788.3 & \textbf{126} & 1463.0 \\
       &  & 4 & 0 & 7.0 & - & - & - & - & 0 & 20718.1 & 0 & 592.3 & 0 & 1506.0 & 99 & 2531.2 & \textbf{108} & 3783.4 \\
       &  & 8 & 0 & 13.2 & - & - & - & - & 0 & 41476.1 & 0 & 876.4 & 0 & 3048.2 & 54 & 4800.8 & \textbf{81} & 8493.9 \\
       &  & 24 & 0 & 39.2 & - & - & - & - & 0 & 71257.8 & 0 & 2612.2 & 0 & 9595.7 & \textbf{9} & 5105.6 & \textbf{9} & 28909.9 \\
      \midrule
      \multirow{8.5}{*}{SST} & \multirow{4}{*}{\shortstack{GPT-2\\Small}} & 2 & 0 & 50.2 & - & - & - & - & 0 & 8122.4 & 24 & 112.5 & 50 & 581.9 & 87 & 790.2 & \textbf{94} & 1962.7 \\
       &  & 4 & 0 & 77.6 & - & - & - & - & 0 & 28785.1 & 0 & 343.9 & 0 & 1307.2 & 67 & 3815.7 & \textbf{81} & 5260.0 \\
       &  & 8 & 0 & 61.1 & - & - & - & - & 0 & 28194.8 & 0 & 649.3 & 1 & 2626.8 & 46 & 7849.6 & \textbf{61} & 11036.0 \\
       &  & 12 & 0 & 68.5 & - & - & - & - & 0 & 41694.3 & 0 & 1022.4 & 0 & 3992.4 & 72 & 12677.0 & \textbf{84} & 17096.6 \\
      \cmidrule{2-19}
       & \multirow{4}{*}{\shortstack{GPT-2\\Medium}} & 2 & 0 & 66.5 & - & - & - & - & 0 & 15233.8 & 42 & 143.5 & 66 & 729.9 & 91 & 1027.2 & \textbf{108} & 2105.2 \\
       &  & 4 & 0 & 70.3 & - & - & - & - & 0 & 27953.7 & 0 & 259.7 & 10 & 1548.2 & 74 & 3018.8 & \textbf{80} & 4919.5 \\
       &  & 8 & 0 & 78.8 & - & - & - & - & 0 & 53430.5 & 1 & 510.1 & 2 & 2955.6 & 56 & 6984.3 & \textbf{68} & 10102.5 \\
       &  & 24 & 0 & 112.2 & - & - & - & - & 0 & 149360.1 & 0 & 1507.0 & 0 & 8472.1 & 24 & 18784.8 & \textbf{25} & 29873.8 \\
      \midrule
      \multicolumn{3}{l}{\textbf{Total}} & 0 & 688.8 & - & - & - & - & 0 & 569409.8 & 139 & 11657.6 & 201 & 46720.8 & 1120 & 83479.4 & \textbf{1339} & 157325.4 \\
      \bottomrule
      \end{tabular}%
      }
    \end{table}

    \noindent\textbf{Benchmark.}
      We evaluate \tool{} on GPT-2-based classifiers
      that we trained on the MNIST and SST datasets, where SST models use IBP training to improve verifiability~\citep{gowal2018effectiveness,xu2024training} (\autoref{app:benchmark}).
      We adopted the exact configuration of GPT-2 model~\citep{radford2019language} and utilize its reference implementation from
      HuggingFace~\footnote{Hugging Face is a widely adopted open-source library that provides standardized reference implementations for SOTA transformer models.
      We specifically use their \texttt{transformers} package~\citep{wolf2020transformers}.}
      In total, we generated 2880 instances across networks and perturbation radii (\autoref{tab:eval-models}).
      For MNIST, image rows become tokens, and we perturb pixels within an $\ell_\infty$ ball.
      For SST, we perturb up to three
      token embeddings within an $\ell_\infty$ ball.

    \noindent\textbf{Baselines.}
      We compare \tool{} with \ibp{}~\citep{wang2018formal} (interval), \deepz{}~\citep{singh2018fast} (zonotope), \abcrown{}~\citep{wang2021beta,zhou2025clipandverify} (polytope),
      and \covenn{}~\citep{duong2025compositional} (polytope).
      Among \tool{}'s variants,
      \toolD{} fuses neither blocks (\emph{\underline{D}}ecomposed),
      \toolA{} fuses only \emph{\underline{A}}ttention (\autoref{sec:fused-attn}),
      \toolM{} fuses only the \emph{\underline{M}}LP (\autoref{sec:structured-mlp}), and
      \toolF{} fuses both (\emph{\underline{F}}ull).
      We exclude \textsc{DeepT}~\citep{bonaert2021fast} and \textsc{PBVerifier}~\citep{huang2026parameterized}
      as they do not support GELU and their post-LayerNorm attention residual (BERT-style) differs from our GPT-based target.
      All methods run on an NVIDIA GeForce RTX 4090 (24 GB VRAM), an AMD Ryzen Threadripper PRO 5975WX (128 GB RAM).

  \subsection{RQ1: Overall Verification Performance}
    \label{sec:rq1}

    \autoref{tab:verification-results} reports verified instances for 2880 instances across MNIST and SST benchmarks.
    \tool{} successfully verifies properties across all network depths, with \toolA{} certifying 46.5\% (1339 of 2880) of all instances.
    In contrast, baseline methods verify none, as they either exceed memory limits (\deepz{}, \abcrown{}) or produce loose bounds (\ibp{}, \covenn{}).

    The results validate our block-specific transformations.
    For the attention block, the fused transform provides the precision necessary to scale to the 24-block GPT-2 Medium.
    On models with more than two blocks, \toolA{} verifies 40.1\% of instances (867 of 2160), whereas unfused-attention variants solve fewer than 1\% (\toolD{} solved 13, \toolM{} solved one of 2160).
    For the MLP block, sequential processing yields higher verified instances than block-level fusion, supporting our analysis in \autoref{sec:unfused-mlp}.
    Computing the exact affine and elementwise GELU layers sequentially increases verified instances by 19.6\% (\toolA{}'s 1339 vs. \toolF{}'s 1120).

    The runtimes reflect the computational cost of maintaining precision.
    Baselines such as \ibp{} and \covenn{} run quickly, but these loose abstractions fail to verify any instances.
    In contrast, \tool{} incurs more computational overhead to compute tighter bounds, \eg, \toolA{} accrues a substantially higher total runtime (157325s) compared to \ibp{} (689s).

  \subsection{RQ2: Fused Transformation Effectiveness}
    \label{sec:rq2}

    To isolate the impact of our fused abstractions, we evaluate performance across varying perturbation radii ($\epsilon \in [10^{-5}, 10^{-3}]$), as shown in~\autoref{fig:verification-cactus}.
    While all baselines verify zero instances across all tested radii, the ablation of \tool{} variants highlights the necessity of the fusion.
    Variants lacking a fused attention block (\toolD{} and \toolM{}) verify only 146 and 125 instances at the tightest radius ($\epsilon = 10^{-5}$) and drop to near zero by $\epsilon = 10^{-4}$.

    Preserving the fused attention allows \toolA{} and \toolF{} to scale to much larger perturbations.
    Although both perform similarly at $\epsilon = 10^{-5}$, \toolA{} performance is consistently better than \toolF{} across all radii.
    In particular, at $\epsilon = 10^{-4}$, \toolA{} verifies 305 instances compared to \toolF{}'s 228, and at $\epsilon = 5 \times 10^{-4}$, \toolA{} certifies 68 instances while \toolF{} drops to 9.
    This highlights the importance of the precision of \toolA{}'s sequential MLP bounds, which are critical as the radius grows.

  \subsection{RQ3: Runtime Analysis}
    \label{sec:rq3}

    \begin{figure}[t]
      \centering
      \begin{subfigure}[t]{0.48\linewidth}
        \centering
        \includegraphics[width=\linewidth]{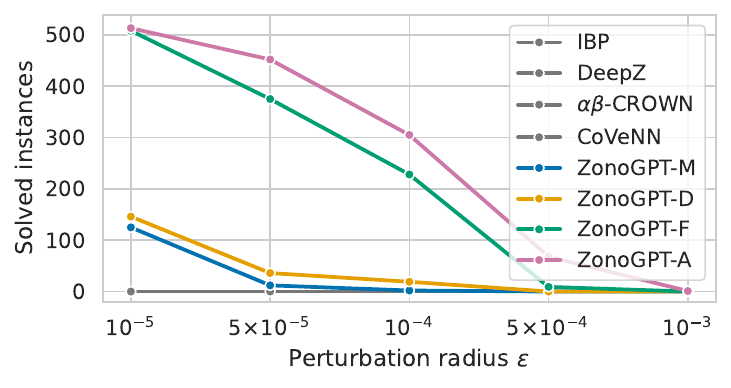}
        \caption{Verified instances by perturbation radius.}
        \label{fig:verification-cactus}
      \end{subfigure}\hfill
      \begin{subfigure}[t]{0.48\linewidth}
        \centering
        \includegraphics[width=\linewidth]{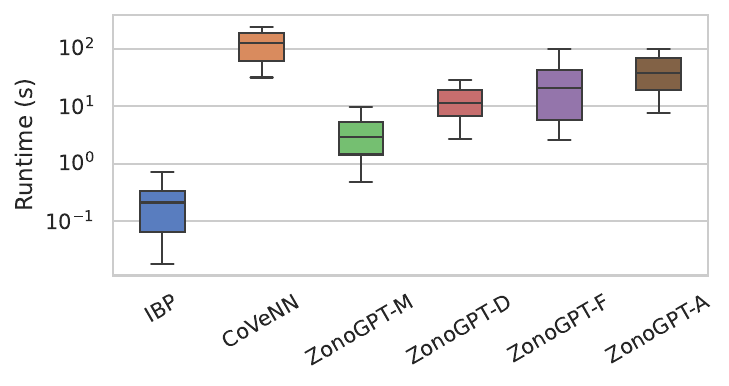}
        \caption{Per-instance runtime (log scale).}
        \label{fig:runtime-boxplot}
      \end{subfigure}
      \caption{Verification and runtime statistics.}
      \label{fig:runtime-comparison}
    \end{figure}

    We evaluate computational efficiency using runtime distributions in~\autoref{fig:runtime-boxplot}.
    \ibp{} finishes instantly (under 1s) by relying on cheap but imprecise interval bounds,
    whereas \covenn{} incurs massive overhead (median $>150$s) due to computationally expensive CROWN analysis within each block.
    Nonetheless, both methods fail to maintain the tight bounds required to scale to deep networks.

    Within the \tool{} variants, runtime predictably scales with abstraction granularity.
    Fusing the MLP block reduces execution time compared to tracking it sequentially,
    making \toolM{} (3s) faster than \toolD{} (11s), and \toolF{} (21s) faster than \toolA{} (38s).
    However, the fused attention abstraction requires slightly more computation to preserve precision,
    \eg, \toolF{}/\toolA{} run slower than \toolM{}/\toolD{}, respectively.
    This is because the fused mechanism for attention must capture interactions across multiple paths
    (including queries, keys, values, and skip connections) into a unified abstract structure.

\section{Conclusion}
  \label{sec:conclusion}
  We presented \tool{}, an abstract domain for verifying large transformers through structured zonotopes and block-specific fused transformations, and successfully verified standard HuggingFace GPT-2 Medium.
  Future work includes jointly over-approximating cached attention states for autoregressive decoding~\citep{pope2023efficiently} and adapting to Diffusion~\citep{peebles2023scalable}.

\bibliographystyle{unsrtnat}
\bibliography{paper}

@inproceedings{katz2017reluplex,
  title={{Reluplex: An efficient SMT solver for verifying deep neural networks}},
  author={Katz, Guy and Barrett, Clark and Dill, David L and Julian, Kyle and Kochenderfer, Mykel J},
  booktitle={International Conference on Computer Aided Verification},
  pages={97--117},
  year={2017},
  organization={Springer},
  doi={10.1007/978-3-319-63387-9_5}
}

@article{singh2018fast,
  title={Fast and effective robustness certification},
  author={Singh, Gagandeep and Gehr, Timon and Mirman, Matthew and P{\"u}schel, Markus and Vechev, Martin},
  journal={Advances in Neural Information Processing Systems},
  volume={31},
  year={2018}
}

@article{singh2019abstract,
  title={An abstract domain for certifying neural networks},
  author={Singh, Gagandeep and Gehr, Timon and P{\"u}schel, Markus and Vechev, Martin},
  journal={Proceedings of the ACM on Programming Languages},
  volume={3},
  number={POPL},
  pages={1--30},
  year={2019},
  publisher={ACM New York, NY, USA}
}

@inproceedings{zhang2018efficient,
  title={Efficient neural network robustness certification with general activation functions},
  author={Zhang, Huan and Weng, Tsui-Wei and Chen, Pin-Yu and Hsieh, Cho-Jui and Daniel, Luca},
  booktitle={Advances in Neural Information Processing Systems},
  volume={31},
  pages={4939--4948},
  year={2018},
  address={Montr{\'e}al, Canada},
  publisher={Curran Associates, Inc.}
}

@article{xu2020automatic,
  title={Automatic perturbation analysis for scalable certified robustness and beyond},
  author={Xu, Kaidi and Shi, Zhouxing and Zhang, Huan and Wang, Yihan and Chang, Kai-Wei and Huang, Minlie and Kailkhura, Bhavya and Lin, Xue and Hsieh, Cho-Jui},
  journal={Advances in Neural Information Processing Systems},
  volume={33},
  pages={1129--1141},
  year={2020}
}

@inproceedings{wang2021beta,
  author = {Wang, Shiqi and Zhang, Huan and Xu, Kaidi and Lin, Xue and Jana, Suman and Hsieh, Cho-Jui and Kolter, J. Zico},
  booktitle = {Advances in Neural Information Processing Systems},
  pages = {29909--29921},
  title = {Beta-CROWN: Efficient Bound Propagation with Per-neuron Split Constraints for Neural Network Robustness Verification},
  volume = {34},
  year = {2021}
}

@inproceedings{duong2025generating,
  title={Generating and Checking DNN Verification Proofs},
  author={Duong, Hai and Nguyen, ThanhVu and Dwyer, Matthew B},
  booktitle={The Thirty-ninth Annual Conference on Neural Information Processing Systems},
  year={2025}
}

@inproceedings{zhou2025clipandverify,
  title={Clip-and-Verify: Linear Constraint-Driven Domain Clipping for Accelerating Neural Network Verification},
  author={Zhou, Duo and Chavez, Jorge and Chen, Hesun and Hanasusanto, Grani A. and Zhang, Huan},
  booktitle={Advances in Neural Information Processing Systems},
  year={2025}
}

@inproceedings{wang2018formal,
  title={Formal security analysis of neural networks using symbolic intervals},
  author={Wang, Shiqi and Pei, Kexin and Whitehouse, Justin and Yang, Junfeng and Jana, Suman},
  booktitle={27th USENIX Security Symposium (USENIX Security 18)},
  pages={1599--1614},
  year={2018},
  url={https://dl.acm.org/doi/10.5555/3277203.3277323}
}

@article{zhang2022general,
  title={General cutting planes for bound-propagation-based neural network verification},
  author={Zhang, Huan and Wang, Shiqi and Xu, Kaidi and Li, Linyi and Li, Bo and Jana, Suman and Hsieh, Cho-Jui and Kolter, J Zico},
  journal={Proceedings of the 36th International Conference on Neural Information Processing Systems},
  year={2022},
  url={https://dl.acm.org/doi/10.5555/3600270.3600391}
}

@inproceedings{cousot1977abstract,
  title={Abstract interpretation: a unified lattice model for static analysis of programs by construction or approximation of fixpoints},
  author={Cousot, Patrick and Cousot, Radhia},
  booktitle={Proceedings of the 4th ACM SIGACT-SIGPLAN symposium on Principles of programming languages},
  pages={238--252},
  year={1977}
}

@inproceedings{ferrari2022complete,
  title={{Complete Verification via Multi-Neuron Relaxation Guided Branch-and-Bound}},
  author={Ferrari, Claudio and Mueller, Mark Niklas and Jovanovi{\'c}, Nikola and Vechev, Martin},
  booktitle={International Conference on Learning Representations},
  year={2022},
  doi={10.48550/arXiv.2205.00263}
}

@article{duong2024harnessing,
    author = {Duong, Hai and Xu, Dong and Nguyen, Thanhvu and Dwyer, Matthew B.},
    title = {Harnessing Neuron Stability to Improve DNN Verification},
    year = {2024},
    publisher = {Association for Computing Machinery},
    address = {New York, NY, USA},
    volume = {1},
    number = {FSE},
    doi = {10.1145/3643765},
    journal = {Proc. ACM Softw. Eng.},
    articleno = {39},
    numpages = {23},
}

@inproceedings{duong2025neuralsat,
  title={NeuralSAT: A High-Performance Verification Tool for Deep Neural Networks},
  author={Duong, Hai and Nguyen, ThanhVu and Dwyer, Matthew B},
  booktitle={International Conference on Computer Aided Verification},
  pages={409--423},
  year={2025},
  organization={Springer}
}

@article{zhou2024scalable,
  title={Scalable Neural Network Verification with Branch-and-bound Inferred Cutting Planes},
  author={Zhou, Duo and Brix, Christopher and Hanasusanto, Grani A and Zhang, Huan},
  journal={arXiv preprint arXiv:2501.00200},
  year={2024}
}

@inproceedings{bak2021nnenum,
  title={{nnenum: Verification of ReLU Neural Networks with Optimized Abstraction Refinement}},
  author={Bak, Stanley},
  booktitle={NASA Formal Methods Symposium},
  pages={19--36},
  year={2021},
  organization={Springer},
  doi={10.1007/978-3-030-76384-8_2}
}

@inproceedings{wu2024marabou,
  title={Marabou 2.0: a versatile formal analyzer of neural networks},
  author={Wu, Haoze and Isac, Omri and Zelji{\'c}, Aleksandar and Tagomori, Teruhiro and Daggitt, Matthew and Kokke, Wen and Refaeli, Idan and Amir, Guy and Julian, Kyle and Bassan, Shahaf and others},
  booktitle={International Conference on Computer Aided Verification},
  pages={249--264},
  year={2024},
  organization={Springer}
}

@article{bunel2020branch,
  title={Branch and bound for piecewise linear neural network verification},
  author={Bunel, Rudy and Mudigonda, P and Turkaslan, Ilker and Torr, P and Lu, Jingyue and Kohli, Pushmeet},
  journal={Journal of Machine Learning Research},
  volume={21},
  number={2020},
  year={2020},
  publisher={Journal of Machine Learning Research},
  url={https://dl.acm.org/doi/10.5555/3455716.3455758}
}

@misc{kaulen20256thinternationalverificationneural,
      title={The 6th International Verification of Neural Networks Competition (VNN-COMP 2025): Summary and Results},
      author={Konstantin Kaulen and Tobias Ladner and Stanley Bak and Christopher Brix and Hai Duong and Thomas Flinkow and Taylor T. Johnson and Lukas Koller and Edoardo Manino and ThanhVu H Nguyen and Haoze Wu},
      year={2025},
      eprint={2512.19007},
      archivePrefix={arXiv},
      primaryClass={cs.LG},
      url={https://arxiv.org/abs/2512.19007},
}

@inproceedings{katz2019marabou,
  title={The marabou framework for verification and analysis of deep neural networks},
  author={Katz, Guy and Huang, Derek A and Ibeling, Duligur and Julian, Kyle and Lazarus, Christopher and Lim, Rachel and Shah, Parth and Thakoor, Shantanu and Wu, Haoze and Zelji{\'c}, Aleksandar and others},
  booktitle={International Conference on Computer Aided Verification},
  pages={443--452},
  year={2019},
  organization={Springer},
  doi={10.1007/978-3-030-25540-4_26}
}

@article{katz2022reluplex,
  title={Reluplex: a calculus for reasoning about deep neural networks},
  author={Katz, Guy and Barrett, Clark and Dill, David L and Julian, Kyle and Kochenderfer, Mykel J},
  journal={Formal Methods in System Design},
  volume={60},
  number={1},
  pages={87--116},
  year={2022},
  publisher={Springer},
  doi={10.1007/s10703-021-00363-7}
}

@inproceedings{duong2025neuralsat2,
  title={Neuralsat: Scaling constraint solving for dnn verification (competition contribution)},
  author={Duong, Hai and Nguyen, ThanhVu},
  booktitle={International Symposium on AI Verification},
  pages={253--259},
  year={2025},
  organization={Springer}
}

@inproceedings{duong2025compositional,
  title={Compositional neural network verification via assume-guarantee reasoning},
  author={Duong, Hai and Shriver, David and Nguyen, ThanhVu and Dwyer, Matthew B},
  booktitle={The Thirty-ninth Annual Conference on Neural Information Processing Systems},
  year={2025}
}

@article{duong2026verifying,
  title={{Verifying Structural Robustness of Deep Neural Network}},
  author = {Duong, Hai and Le, Thanh and Nguyen, Lam and Nguyen, ThanhVu},
  journal={Proceedings of the ACM on Software Engineering},
  volume={3},
  number={FSE},
  year={2026}
}

@inproceedings{duong2026verifying2,
  title={Verifying Neural Network Robustness with Dual Perturbations},
  author={Duong, Hai and Nguyen, Lam and Le, Thanh and Nguyen, ThanhVu},
  booktitle={Proceedings of the IEEE/CVF Conference on Computer Vision and Pattern Recognition},
  pages={27916--27925},
  year={2026}
}

@article{shi2020robustness,
  title={Robustness verification for transformers},
  author={Shi, Zhouxing and Zhang, Huan and Chang, Kai-Wei and Huang, Minlie and Hsieh, Cho-Jui},
  journal={arXiv preprint arXiv:2002.06622},
  year={2020}
}

@article{yang2025qwen3,
  title={Qwen3 technical report},
  author={Yang, An and Li, Anfeng and Yang, Baosong and Zhang, Beichen and Hui, Binyuan and Zheng, Bo and Yu, Bowen and Gao, Chang and Huang, Chengen and Lv, Chenxu and others},
  journal={arXiv preprint arXiv:2505.09388},
  year={2025}
}

@inproceedings{huang2024position,
  title={Position: Trustllm: Trustworthiness in large language models},
  author={Huang, Yue and Sun, Lichao and Wang, Haoran and Wu, Siyuan and Zhang, Qihui and Li, Yuan and Gao, Chujie and Huang, Yixin and Lyu, Wenhan and Zhang, Yixuan and others},
  booktitle={International Conference on Machine Learning},
  pages={20166--20270},
  year={2024},
  organization={PMLR}
}

@inproceedings{tran2019star,
  title={Star-based reachability analysis of deep neural networks},
  author={Tran, Hoang-Dung and Manzanas Lopez, Diago and Musau, Patrick and Yang, Xiaodong and Nguyen, Luan Viet and Xiang, Weiming and Johnson, Taylor T},
  booktitle={International symposium on formal methods},
  pages={670--686},
  year={2019},
  organization={Springer}
}

@article{pope2023efficiently,
  title={Efficiently scaling transformer inference},
  author={Pope, Reiner and Douglas, Sholto and Chowdhery, Aakanksha and Devlin, Jacob and Bradbury, James and Heek, Jonathan and Xiao, Kefan and Agrawal, Shivani and Dean, Jeff},
  journal={Proceedings of machine learning and systems},
  volume={5},
  pages={606--624},
  year={2023}
}

@inproceedings{peebles2023scalable,
  title={Scalable diffusion models with transformers},
  author={Peebles, William and Xie, Saining},
  booktitle={2023 IEEE/CVF International Conference on Computer Vision (ICCV)},
  pages={4172--4182},
  year={2023},
  organization={IEEE}
}

@article{vaswani2017attention,
  title={Attention is all you need},
  author={Vaswani, Ashish and Shazeer, Noam and Parmar, Niki and Uszkoreit, Jakob and Jones, Llion and Gomez, Aidan N and Kaiser, {\L}ukasz and Polosukhin, Illia},
  journal={Advances in neural information processing systems},
  volume={30},
  year={2017}
}

@article{radford2019language,
  title={Language models are unsupervised multitask learners},
  author={Radford, Alec and Wu, Jeffrey and Child, Rewon and Luan, David and Amodei, Dario and Sutskever, Ilya and others},
  journal={OpenAI blog},
  volume={1},
  number={8},
  pages={9},
  year={2019}
}

@article{zou2023universal,
  title={Universal and transferable adversarial attacks on aligned language models},
  author={Zou, Andy and Wang, Zifan and Carlini, Nicholas and Nasr, Milad and Kolter, J Zico and Fredrikson, Matt},
  journal={arXiv preprint arXiv:2307.15043},
  year={2023}
}

@inproceedings{bonaert2021fast,
  title={Fast and precise certification of transformers},
  author={Bonaert, Gregory and Dimitrov, Dimitar I and Baader, Maximilian and Vechev, Martin},
  booktitle={Proceedings of the 42nd ACM SIGPLAN international conference on programming language design and implementation},
  pages={466--481},
  year={2021}
}

@inproceedings{wei2023convex,
  title={Convex bounds on the softmax function with applications to robustness verification},
  author={Wei, Dennis and Wu, Haoze and Wu, Min and Chen, Pin-Yu and Barrett, Clark and Farchi, Eitan},
  booktitle={International Conference on Artificial Intelligence and Statistics},
  pages={6853--6878},
  year={2023},
  organization={PMLR}
}

@inproceedings{zhang2024galileo,
  title={Galileo: General linear relaxation framework for tightening robustness certification of transformers},
  author={Zhang, Yunruo and Shen, Lujia and Guo, Shanqing and Ji, Shouling},
  booktitle={Proceedings of the AAAI Conference on Artificial Intelligence},
  volume={38},
  number={19},
  pages={21797--21805},
  year={2024}
}

@inproceedings{huang2026parameterized,
  title={Parameterized abstract interpretation for transformer verification},
  author={Huang, Pei and Wei, Dennis and Isac, Omri and Wu, Haoze and Wu, Min and Barrett, Clark},
  booktitle={Proceedings of the AAAI Conference on Artificial Intelligence},
  volume={40},
  number={42},
  pages={35500--35508},
  year={2026}
}

@article{liu2024deepseek,
  title={Deepseek-v3 technical report},
  author={Liu, Aixin and Feng, Bei and Xue, Bing and Wang, Bingxuan and Wu, Bochao and Lu, Chengda and Zhao, Chenggang and Deng, Chengqi and Zhang, Chenyu and Ruan, Chong and others},
  journal={arXiv preprint arXiv:2412.19437},
  year={2024}
}

@article{comanici2025gemini,
  title={Gemini 2.5: Pushing the frontier with advanced reasoning, multimodality, long context, and next generation agentic capabilities},
  author={Comanici, Gheorghe and Bieber, Eric and Schaekermann, Mike and Pasupat, Ice and Sachdeva, Noveen and Dhillon, Inderjit and Blistein, Marcel and Ram, Ori and Zhang, Dan and Rosen, Evan and others},
  journal={arXiv preprint arXiv:2507.06261},
  year={2025}
}

@inproceedings{chu2025jailbreakradar,
  title={JailbreakRadar: Comprehensive assessment of jailbreak attacks against LLMs},
  author={Chu, Junjie and Liu, Yugeng and Yang, Ziqing and Shen, Xinyue and Backes, Michael and Zhang, Yang},
  booktitle={Proceedings of the 63rd Annual Meeting of the Association for Computational Linguistics (Volume 1: Long Papers)},
  pages={21538--21566},
  year={2025}
}

@article{zhang2025evaluating,
  title={Evaluating and improving robustness in large language models: a survey and future directions},
  author={Zhang, Kun and Wu, Le and Yu, Kui and Lv, Guangyi and Zhang, Dacao},
  journal={arXiv preprint arXiv:2506.11111},
  year={2025}
}

@inproceedings{joo2025harmful,
  title={Harmful prompt laundering: Jailbreaking LLMs with abductive styles and symbolic encoding},
  author={Joo, Seongho and Koh, Hyukhun and Jung, Kyomin},
  booktitle={Proceedings of the 2025 Conference on Empirical Methods in Natural Language Processing},
  pages={25500--25535},
  year={2025}
}

@article{wang2023decodingtrust,
  title={Decodingtrust: A comprehensive assessment of trustworthiness in $\{$GPT$\}$ models},
  author={Wang, Boxin and Chen, Weixin and Pei, Hengzhi and Xie, Chulin and Kang, Mintong and Zhang, Chenhui and Xu, Chejian and Xiong, Zidi and Dutta, Ritik and Schaeffer, Rylan and others},
  year={2023},
  publisher={Neural Information Processing Systems Datasets; Benchmarks Track}
}

@article{gallegos2024bias,
  title={Bias and fairness in large language models: A survey},
  author={Gallegos, Isabel O and Rossi, Ryan A and Barrow, Joe and Tanjim, Md Mehrab and Kim, Sungchul and Dernoncourt, Franck and Yu, Tong and Zhang, Ruiyi and Ahmed, Nesreen K},
  journal={Computational linguistics},
  volume={50},
  number={3},
  pages={1097--1179},
  year={2024},
  publisher={MIT Press 255 Main Street, 9th Floor, Cambridge, Massachusetts 02142, USA~…}
}

@inproceedings{jia2019certified,
  title={Certified robustness to adversarial word substitutions},
  author={Jia, Robin and Raghunathan, Aditi and G{\"o}ksel, Kerem and Liang, Percy},
  booktitle={Proceedings of the 2019 conference on empirical methods in natural language processing and the 9th international joint conference on natural language processing (EMNLP-IJCNLP)},
  pages={4129--4142},
  year={2019}
}

@article{zhu2024promptbench,
  title={Promptbench: A unified library for evaluation of large language models},
  author={Zhu, Kaijie and Zhao, Qinlin and Chen, Hao and Wang, Jindong and Xie, Xing},
  journal={Journal of Machine Learning Research},
  volume={25},
  number={254},
  pages={1--22},
  year={2024}
}

@inproceedings{wolf2020transformers,
  title={Transformers: State-of-the-art natural language processing},
  author={Wolf, Thomas and Debut, Lysandre and Sanh, Victor and Chaumond, Julien and Delangue, Clement and Moi, Anthony and Cistac, Pierric and Rault, Tim and Louf, R{\'e}mi and Funtowicz, Morgan and others},
  booktitle={Proceedings of the 2020 conference on empirical methods in natural language processing: system demonstrations},
  pages={38--45},
  year={2020}
}

@inproceedings{xu2024anllm,
 author = {Xu, Xilie and Kong, Keyi and Liu, Ning and Cui, Lizhen and Wang, Di and Zhang, Jingfeng and Kankanhalli, Mohan},
 booktitle = {International Conference on Learning Representations},
 editor = {B. Kim and Y. Yue and S. Chaudhuri and K. Fragkiadaki and M. Khan and Y. Sun},
 pages = {2900--2922},
 title = {An LLM can Fool Itself: A Prompt-Based Adversarial Attack},
 volume = {2024},
 year = {2024}
}

@article{duong2023dpll,
  title={A dpll (t) framework for verifying deep neural networks},
  author={Duong, Hai and Nguyen, ThanhVu and Dwyer, Matthew},
  journal={arXiv preprint arXiv:2307.10266},
  year={2023}
}

@inproceedings{xu2024training,
  title={Training for verification: Increasing neuron stability to scale dnn verification},
  author={Xu, Dong and Mozumder, Nusrat Jahan and Duong, Hai and Dwyer, Matthew B},
  booktitle={International Conference on Tools and Algorithms for the Construction and Analysis of Systems},
  pages={24--44},
  year={2024},
  organization={Springer}
}

@article{duong2026verifying3,
  title={Verifying Neural Networks with Reinforcement Learning},
  author={Duong, Hai and Thanh, Le and Nguyen, ThanhVu},
  journal={accepted at Advances in Neural Information Processing Systems},
  year={2026}
}

@article{le2025fggm,
  title={Fggm: Formal grey-box gradient method for attacking drl-based mu-mimo scheduler},
  author={Le, Thanh and Duong, Hai and Ji, Yusheng and Nguyen, ThanhVu and Lui, John},
  journal={arXiv preprint arXiv:2510.26075},
  year={2025}
}

@article{le2026verifying,
  title={Verifying DNN-based Semantic Communication Against Generative Adversarial Noise},
  author={Le, Thanh and Duong, Hai and Nguyen, ThanhVu and Matsumura, Takeshi},
  journal={arXiv preprint arXiv:2602.08801},
  year={2026}
}

@article{gowal2018effectiveness,
  title={On the Effectiveness of Interval Bound Propagation for Training Verifiably Robust Models},
  author={Gowal, Sven and Dvijotham, Krishnamurthy and Stanforth, Robert and Bunel, Rudy and Qin, Chongli and Uesato, Jonathan and Arandjelovic, Relja and Mann, Timothy and Kohli, Pushmeet},
  journal={arXiv preprint arXiv:1810.12715},
  year={2018}
}

\appendix

\section{GPT Transformations}

\subsection{Fused LayerNorm Transform}
  \label{app:layernorm}

  \begin{theorem}[LayerNorm soundness]
    \label{thm:layernorm-soundness}
    For $\Pi=I-\mathbf1\mathbf1^\top/d$ and $\tau>0$, LayerNorm is
    \begin{equation}
      \operatorname{LN}(x)
      =\gamma\odot\frac{\Pi x}{\sqrt{\tau+\|\Pi x\|_2^2/d}}+\beta.
    \end{equation}
    For $Z=\mathcal Z(c,G,L,0)$, let $H_j$ range over its shared and embedded local generators.
    \tool{} computes nonnegative vectors $r_N$ and $R_N$ such that every $x\in Z$ satisfies
    \begin{equation}
      |J_{\rm LN}(c)(x-c)|\le r_N,
      \qquad
      |\operatorname{LN}(x)-\operatorname{LN}(c)-J_{\rm LN}(c)(x-c)|\le R_N.
      \label{eq:app-ln-required-bounds}
    \end{equation}
    The resulting zonotope
    \begin{equation}
      Z_N=\mathcal Z\!\left(
      \operatorname{LN}(c),
      \{J_{\rm LN}(c)G_j\}_j,
      \{J_{\rm LN}(c_s)L_s\}_{s=1}^{S},
      R_N\right)
      \label{eq:app-ln-sound-output}
    \end{equation}
    satisfies $\operatorname{LN}(Z)\subseteq Z_N$.
  \end{theorem}

  \begin{proof}
    LayerNorm acts independently on each token, so the proof considers one token.
    For $x\in Z$, its change from the center and the corresponding centered change are
    \begin{equation}
      \Delta=x-c=\sum_jH_j\epsilon_j,
      \qquad
      \bar\Delta=\Pi\Delta=\sum_j\underbrace{\Pi H_j}_{\bar H_j}\epsilon_j.
      \label{eq:app-ln-input-change}
    \end{equation}
    The centered input and its squared norm at $c$ are
    \begin{equation}
      z_0=\Pi c,
      \qquad
      a_0=\tau+\|z_0\|_2^2/d.
    \end{equation}
    Consequently, LayerNorm at $x=c+\Delta$ is
    \begin{equation}
      \operatorname{LN}(c+\Delta)
      =\gamma\odot(z_0+\bar\Delta)
      \rho(a_0+\delta a_{\rm lin}+\delta a_{\rm quad})+\beta,
      \label{eq:app-ln-expanded}
    \end{equation}
    where
    \begin{equation}
      \rho(a)=a^{-1/2},
      \qquad
      \delta a_{\rm lin}=2z_0^\top\bar\Delta/d,
      \qquad
      \delta a_{\rm quad}=\|\bar\Delta\|_2^2/d.
      \label{eq:app-ln-norm-change}
    \end{equation}

    \paragraph{Linear change computation.}
    Differentiating \autoref{eq:app-ln-expanded} at $\Delta=0$ gives
    \begin{equation}
      J_{\rm LN}(c)\Delta
      =\gamma\odot\left(
      \bar\Delta\rho(a_0)
      +z_0\rho'(a_0)\delta a_{\rm lin}
      \right).
      \label{eq:app-ln-jacobian}
    \end{equation}
    Because the expression is linear in $\Delta$, substituting \autoref{eq:app-ln-input-change} into \autoref{eq:app-ln-jacobian} gives
    \begin{equation}
      J_{\rm LN}(c)\Delta
      =\sum_jJ_{\rm LN}(c)H_j\epsilon_j.
    \end{equation}
    Thus \tool{} computes
    \begin{equation}
      |J_{\rm LN}(c)\Delta| = |J_{\rm LN}(c)(x - c)| \le r_N,\qquad \qquad
      r_N=\sum_j|J_{\rm LN}(c)H_j|,
      \qquad
      \label{eq:app-ln-linear-bound}
    \end{equation}
    where the inequality follows from $|\epsilon_j|\le1$ and the triangle inequality.

    \paragraph{Nonlinear error computation.}
    Subtracting $\operatorname{LN}(c)$ and $J_{\rm LN}(c)\Delta$ from \autoref{eq:app-ln-expanded} gives exactly
    \begin{equation}
      \operatorname{LN}(c+\Delta)-\operatorname{LN}(c)-J_{\rm LN}(c)\Delta
      =\gamma\odot\left(z_0e_\rho+\bar\Delta d_\rho\right),
      \label{eq:app-ln-error-decomposition}
    \end{equation}
    where
    \begin{equation}
      \begin{aligned}
        e_\rho
        &=\rho(a_0+\delta a_{\rm lin}+\delta a_{\rm quad})
        -\rho(a_0)-\rho'(a_0)\delta a_{\rm lin},\\
        d_\rho
        &=\rho(a_0+\delta a_{\rm lin}+\delta a_{\rm quad})-\rho(a_0).
      \end{aligned}
      \label{eq:app-ln-scalar-errors}
    \end{equation}

    The input generators bound the quantities in \autoref{eq:app-ln-scalar-errors} by
    \begin{equation}
      \begin{aligned}
        r_{\rm cen}&=\sum_j|\bar H_j|,
        &|\bar\Delta|&\le r_{\rm cen},\\
        r_{\rm lin}&=\sum_j|2z_0^\top\bar H_j/d|,
        &|\delta a_{\rm lin}|&\le r_{\rm lin},\\
        r_{\rm quad}&=\|r_{\rm cen}\|_2^2/d,
        &0\le\delta a_{\rm quad}&\le r_{\rm quad},\\
        r_a&=r_{\rm lin}+r_{\rm quad},
        &|\delta a_{\rm lin}+\delta a_{\rm quad}|&\le r_a.
      \end{aligned}
      \label{eq:app-ln-generator-bounds}
    \end{equation}
    Taylor's theorem also requires a positive lower bound for the argument of $\rho$.
    For $\widehat z=z_0/\|z_0\|_2$ when $z_0\ne0$ and $\widehat z=0$ otherwise, the generator bounds give
    \begin{equation}
      \begin{aligned}
        r_{\parallel}&=\sum_j|\widehat z^\top\bar H_j|,\\
        a_{\min}&=\tau+\max\left\{
        \max(0,\|z_0\|_2^2/d-r_{\rm lin}),
        (\|z_0\|_2-r_{\parallel})_+^2/d
        \right\}.
      \end{aligned}
      \label{eq:app-ln-denominator-lower-bound}
    \end{equation}
    This construction ensures
    \begin{equation}
      \tau+\|z_0+t\bar\Delta\|_2^2/d\ge a_{\min}
      \qquad\text{for every }t\in[0,1].
    \end{equation}

    Using $\rho'(a)=-\tfrac12a^{-3/2}$ and $\rho''(a)=\tfrac34a^{-5/2}$, Taylor's theorem and the mean value theorem yield
    \begin{equation}
      \begin{aligned}
        |e_\rho|
        &\le\tfrac12a_0^{-3/2}r_{\rm quad}
        +\tfrac38a_{\min}^{-5/2}r_a^2,\\
        |d_\rho|
        &\le\tfrac12a_{\min}^{-3/2}r_a.
      \end{aligned}
      \label{eq:app-ln-scalar-error-bounds}
    \end{equation}
    Applying these inequalities and $|\bar\Delta|\le r_{\rm cen}$ to \autoref{eq:app-ln-error-decomposition} gives
    \begin{equation}
      R_N=|\gamma|\odot\left[
      |z_0|\left(
      \tfrac12a_0^{-3/2}r_{\rm quad}
      +\tfrac38a_{\min}^{-5/2}r_a^2
      \right)
      +\tfrac12r_{\rm cen}a_{\min}^{-3/2}r_a
      \right].
      \label{eq:method-ln-radius}
    \end{equation}
    Therefore, \autoref{eq:app-ln-linear-bound} gives the first inequality in \autoref{eq:app-ln-required-bounds}, and \autoref{eq:method-ln-radius} gives the second.
    For every $x\in Z=\mathcal Z(c,G,L,0)$, \autoref{eq:app-ln-required-bounds} gives
    \begin{equation}
      \begin{aligned}
        \operatorname{LN}(x)
        &=\operatorname{LN}(c)
        +\sum_jJ_{\rm LN}(c)G_j\epsilon_j\\
        &\quad+\sum_{s=1}^{S}\operatorname{emb}_s\!\left(
        J_{\rm LN}(c_s)L_s\epsilon_s^{\rm loc}\right)+e_N,
        \qquad |e_N|\le R_N
      \end{aligned}
      \label{eq:app-ln-output-expansion}
    \end{equation}
    Therefore,
    \begin{equation}
      \operatorname{LN}(Z)\subseteq
      Z_N=\mathcal Z\!\left(
      \operatorname{LN}(c),
      \{J_{\rm LN}(c)G_j\}_j,
      \{J_{\rm LN}(c_s)L_s\}_{s=1}^{S},
      R_N\right)
      \label{eq:app-ln-output-zonotope}
    \end{equation}
    The Jacobian images retain every input symbol, while only the nonlinear error $e_N$ enters the interval.
  \end{proof}

  \begin{remark}[Lean mechanization of LayerNorm soundness]
    \label{rem:lean-layernorm}
    \autoref{thm:layernorm-soundness} is fully mechanized in Lean~4 as \texttt{layerNorm\_image\_subset\_lnZN} in \texttt{ZonoGpt/LayerNorm.lean}, whose only hypothesis is $\tau>0$.
    The mechanization follows the proof above.
    It first shows that the closed-form Jacobian $J_{\rm LN}(c)$ is the Fréchet derivative of LayerNorm (\texttt{ZonoGpt/LayerNormDeriv.lean}).
    For a single token, it then bounds the first-order change by $r_N$ and the Taylor remainder by $R_N$, using second-order Taylor and mean-value bounds for $a\mapsto a^{-1/2}$ (\texttt{ZonoGpt/Rho.lean}, \texttt{ZonoGpt/Calculus.lean}); the hypothesis $\tau>0$ serves only to guarantee $a_{\min}>0$.
    Lifting these bounds to all tokens yields the symbol-preserving expansion \autoref{eq:app-ln-output-expansion} (\texttt{layerNorm\_szono\_expansion}), from which the inclusion $\operatorname{LN}(Z)\subseteq Z_N$ follows directly.
    The proofs proceed step by step, mainly using the tactics \texttt{have}, \texttt{rw}, \texttt{simp}, and \texttt{ring}, with \texttt{linarith} and \texttt{positivity} closing the inequalities. They contain no \texttt{sorry} and introduce no axioms.
  \end{remark}

\subsection{Affine GELU Transform}
  \label{app:gelu-affine}

  The MLP requires bounds on tanh-GELU
  \begin{equation}
    g(t)=\tfrac12t\left(1+\tanh h(t)\right),\qquad
    h(t)=\sqrt{2/\pi}\,(t+0.044715t^3)
    \label{eq:app-gelu-definition}
  \end{equation}
  over each preactivation interval $[\ell,u]$.
  \tool{} bounds $g$ on $[\ell,u]$ by a linear lower bound $at+r^{\rm lb}$ and a linear upper bound $at+r^{\rm ub}$ that share the coefficient $a$, where $r^{\rm lb}$ and $r^{\rm ub}$ bound the residual $q(t)=g(t)-at$ (\autoref{eq:app-gelu-residual-bound}).
  These bounds hold for every $a$; for $\ell<u$, \tool{} uses
  \begin{equation}
    a=\frac{g(u)-g(\ell)}{u-\ell}.
    \label{eq:app-gelu-slope}
  \end{equation}

  The second derivative of $g$ satisfies
  \begin{equation}
    |g''(t)|
    \le \operatorname{sech}^2(h(t))
    \left(
    |h'(t)|+|t|\left(|h'(t)|^2+\tfrac12|h''(t)|\right)
    \right),
    \label{eq:app-gelu-curvature}
  \end{equation}
  where $h'(t)=\sqrt{2/\pi}\,(1+0.134145t^2)$ and $h''(t)=0.26829\sqrt{2/\pi}\,t$.
  Partition $[\ell,u]$ into $\ell=t_0<\cdots<t_K=u$, and let $C_i$ bound the right-hand side of \autoref{eq:app-gelu-curvature} on $[t_i,t_{i+1}]$.
  Linear interpolation of $q$ on this interval gives
  \begin{equation}
    \begin{aligned}
      q_i^{\rm lb}&=\min\{q(t_i),q(t_{i+1})\}
      -\frac{C_i(t_{i+1}-t_i)^2}{8},\\
      q_i^{\rm ub}&=\max\{q(t_i),q(t_{i+1})\}
      +\frac{C_i(t_{i+1}-t_i)^2}{8}
    \end{aligned}
    \label{eq:app-gelu-partition-bound}
  \end{equation}
  and therefore $q(t)\in[r_{\rm part}^{\rm lb},r_{\rm part}^{\rm ub}]$, where
  \begin{equation}
    r_{\rm part}^{\rm lb}=\min_i q_i^{\rm lb},\qquad
    r_{\rm part}^{\rm ub}=\max_i q_i^{\rm ub}
    \label{eq:app-gelu-partition-union}
  \end{equation}

  \tool{} also uses the global relation $g(t)=\operatorname{ReLU}(t)-e_{\rm relu}(t)$, where $e_{\rm relu}(t)=|t|/\bigl(1+e^{2h(|t|)}\bigr)$ because $h$ is odd, and
  \begin{equation}
    0\le e_{\rm relu}(t)\le e_{\max}.
    \label{eq:app-gelu-relu-gap}
  \end{equation}
  \tool{} computes $e_{\max}\approx0.1703$ numerically.
  For $|t|\le8$, it splits $[0,8]$ into $8192$ equal pieces and bounds $e_{\rm relu}$ on each piece $[t_-,t_+]$ by $t_+/\bigl(1+e^{2h(t_-)}\bigr)$, since $|t|$ increases and $1/\bigl(1+e^{2h(|t|)}\bigr)$ decreases in $|t|$.
  For $|t|\ge8$, it uses $e_{\rm relu}(t)\le8e^{-2h(8)}$, since $|t|e^{-2h(|t|)}$ decreases there.
  This computation runs in floating point with a small outward margin rather than directed rounding.
  A coarser closed-form value follows from $h(|t|)\ge\sqrt{2/\pi}\,|t|$: $e_{\rm relu}(t)\le|t|e^{-2\sqrt{2/\pi}\,|t|}\le\sqrt{\pi/2}/(2e)\approx0.231$.
  Because $\operatorname{ReLU}(t)-at$ is piecewise affine, its extrema on $[\ell,u]$ occur at $\ell$, $u$, or zero when the interval contains zero.
  For this set of points $E$, define
  \begin{equation}
    r_{\rm global}^{\rm lb}=\min_{t\in E}\bigl(\operatorname{ReLU}(t)-at\bigr)-e_{\max},\qquad
    r_{\rm global}^{\rm ub}=\max_{t\in E}\bigl(\operatorname{ReLU}(t)-at\bigr)
    \label{eq:app-gelu-global-bound}
  \end{equation}
  Intersecting the two sound bounds gives
  \begin{equation}
    r^{\rm lb}=\max\{r_{\rm part}^{\rm lb},r_{\rm global}^{\rm lb}\},\qquad
    r^{\rm ub}=\min\{r_{\rm part}^{\rm ub},r_{\rm global}^{\rm ub}\}
    \label{eq:app-gelu-residual-bound}
  \end{equation}
  The implementation uses $K=32$ equal subintervals.

  Set $m=(r^{\rm lb}+r^{\rm ub})/2$ and $d=(r^{\rm ub}-r^{\rm lb})/2$ elementwise.
  Then every $t\in[\ell,u]$ satisfies
  \begin{equation}
    g(t)=at+m+d\epsilon^g,\qquad |\epsilon^g|\le1
    \label{eq:app-gelu-affine-bound}
  \end{equation}
  For a preactivation zonotope $Z_z=\mathcal Z(c_z,G_z,L_z,b_z)$, applying \autoref{eq:app-gelu-affine-bound} elementwise produces
  \begin{equation}
    \begin{aligned}
      Z_g&=\mathcal Z(c_g,G_g,L_g,b_g),
      &c_g&=a\odot c_z+m,
      &(G_g)_j&=a\odot(G_z)_j,\\
      (L_g)_s&=[\,a\odot(L_z)_s\mid\operatorname{diag}(d_s)\,],
      &(b_g)_s&=|a_s|\odot(b_z)_s
    \end{aligned}
    \label{eq:app-gelu-output-zonotope}
  \end{equation}
  The new columns $\operatorname{diag}(d_s)$ are local because GELU acts independently on every token and coordinate.

  \begin{theorem}[Affine GELU soundness]
    \label{thm:gelu-soundness}
    Let $Z_z=\mathcal Z(c_z,G_z,L_z,b_z)$ be a GELU input zonotope, and let $[\ell,u]$ contain every coordinate of $Z_z$.
    The bounds in \autoref{eq:app-gelu-residual-bound} and the output in \autoref{eq:app-gelu-output-zonotope} satisfy
    \begin{equation}
      \{g(z):z\in Z_z\}\subseteq Z_g.
      \label{eq:app-gelu-soundness}
    \end{equation}
    Moreover, the output keeps the input noise symbols: every $z=c_z+\sum_{j=1}^m(G_z)_j\epsilon_j+\sum_{s=1}^S\operatorname{emb}_s\bigl((L_z)_s\epsilon^{\rm loc}_s\bigr)+e\in Z_z$ satisfies $g(z)=c_g+\sum_{j=1}^m(G_g)_j\epsilon_j+\sum_{s=1}^S\operatorname{emb}_s\bigl((L_g)_s(\epsilon^{\rm loc}_s,\epsilon^g_s)\bigr)+a\odot e$ for some $\|\epsilon^g\|_\infty\le1$.
  \end{theorem}

  \begin{proof}[Proof of \autoref{thm:gelu-soundness}]
    On each partition interval, the linear-interpolation error is at most $C_i(t_{i+1}-t_i)^2/8$.
    The interpolant lies between its endpoint values, which proves \autoref{eq:app-gelu-partition-bound}.
    The identity $g=\operatorname{ReLU}-e_{\rm relu}$ with \autoref{eq:app-gelu-relu-gap} proves \autoref{eq:app-gelu-global-bound}.
    Their intersection contains $q(t)$, so centering $[r^{\rm lb},r^{\rm ub}]$ gives \autoref{eq:app-gelu-affine-bound}.
    Applying this scalar identity to every coordinate of $z$ gives the symbol-preserving form with the fields in \autoref{eq:app-gelu-output-zonotope}, where $|a\odot e|\le|a|\odot b_z=b_g$.
    Therefore, every $g(z)$ belongs to $Z_g$, which proves \autoref{eq:app-gelu-soundness}.
  \end{proof}

  \begin{remark}[Lean mechanization of affine GELU soundness]
    \label{rem:lean-gelu}
    \autoref{thm:gelu-soundness} is mechanized as \texttt{gelu\_soundness} in \texttt{ZonoGpt/GeluSound.lean}.
    The mechanized bound holds for any $K\ge1$ partition nodes that cover $[\ell,u]$; the implementation's $K=32$ equal subintervals and the coefficient $a$ in~\eqref{eq:app-gelu-slope} serve only to tighten the linear lower and upper bounds.
    The symbol-preserving statement of \autoref{thm:gelu-soundness} is proved as \texttt{gelu\_szono\_symbolic}, which skip connections require because they share noise symbols with earlier generators (\texttt{set\_inclusion\_not\_compositional}).
    The mechanization follows the GELU bound of \autoref{app:gelu-affine} step by step.
    It first verifies the formulas for $g'$ and $g''$, the curvature bound~\eqref{eq:app-gelu-curvature}, and the closed-form ReLU gap bound $e_{\max}=\sqrt{\pi/2}/(2e)$ (\texttt{ZonoGpt/GeluCalc.lean}).
    On each subinterval, it bounds the residual $q(t)=g(t)-at$ by the linear interpolation error in~\eqref{eq:app-gelu-partition-bound}; combining these bounds with the global ReLU bound gives the residual interval $[r^{\rm lb},r^{\rm ub}]$, and hence an affine bound that keeps the noise symbols of the input.
    The proofs proceed step by step, mainly using the tactics \texttt{have}, \texttt{rw}, \texttt{exact}, and \texttt{unfold}, with \texttt{linarith} and \texttt{positivity} closing the inequalities and \texttt{split\_ifs} handling the case analysis in the $\min$ and $\max$ bounds.
  \end{remark}

\subsection{Fused MLP Transform}\label{sec:structured-mlp}
  For input $x$, the MLP block computes $M(x)=x+W_2g(W_1\operatorname{LN}(x)+b_1)+b_2$.
  To construct output bounds for $M$, \tool{} computes the three terms of \autoref{eq:method-fused-construction} for the complete MLP residual: $M(c)$, $J_M(c)H_j$, and $R_M$.
  First, \tool{} computes $M(c)$ by evaluating the MLP residual at $x=c$.
  Second, \tool{} computes $J_M(c)H_j$ along the path $x(t)=c+tH_j$ from \autoref{eq:attention-direction}:
  \begin{equation}
    \begin{aligned}
      &n=\operatorname{LN}(c), \qquad z_c=W_1n+b_1, \qquad
      \dot n=J_{\rm LN}(c)H_j, \qquad \dot z=W_1\dot n \\
      &J_M(c)H_j=H_j+W_2\bigl(g'(z_c)\odot\dot z\bigr)=H_j+J_B(c)H_j
    \end{aligned}
    \label{eq:method-mlp-jvp}
  \end{equation}

  Third, \tool{} computes the nonlinear-error bound $R_M$ by \autoref{thm:mlp-remainder}.
  \begin{theorem}[MLP nonlinear-error bound]
    \label{thm:mlp-remainder}
    Let $M(x)=x+W_2g(W_1\operatorname{LN}(x)+b_1)+b_2$ act tokenwise, let $Z=\{c+\sum_jH_j\epsilon_j:|\epsilon_j|\le1\}$, and set $z_c=W_1\operatorname{LN}(c)+b_1$.
    Suppose a nonnegative tuple $(r_N,R_N,C_g)$ satisfies $|J_{\rm LN}(c)(x-c)|\le r_N$ and $|\operatorname{LN}(x)-\operatorname{LN}(c)-J_{\rm LN}(c)(x-c)|\le R_N$ for every $x\in Z$.
    Suppose $|g''(u)|\le C_g$ elementwise for every $u$ between $z_c-|W_1|(r_N+R_N)$ and $z_c+|W_1|(r_N+R_N)$.
    Then, for every $x\in Z$,
    \begin{equation}
      \underbrace{|M(x)-M(c)-J_M(c)(x-c)|}_{e_M} \le \underbrace{|W_2\operatorname{diag}(g'(z_c))W_1|R_N+\tfrac12|W_2|\left(C_g\odot\bigl(|W_1|(r_N+R_N)\bigr)^2\right)}_{R_M}
      \label{eq:mlp-remainder-bound}
    \end{equation}
  \end{theorem}
  \begin{proof}[Proof sketch]
    Taylor's theorem splits the omitted error into a transported LayerNorm error and a GELU curvature term.
    \autoref{app:mlp-remainder} gives the computation and full proof.
  \end{proof}

  Finally, \tool{} computes the MLP output zonotope $Z_M=\mathcal Z(c_M,G_M,L_M,b_M)$ (\autoref{def:method-domain}):
  \begin{equation}
    \label{eq:mlp-zonotope-coefficients}
    c_M=M(c), \quad
    (G_M)_j=J_M(c)G_j,\quad
    (L_M)_s=[\,(J_M(c)\operatorname{emb}_s(L_{s,k}))_s\,]_{k=1}^{q_s}, \quad
    (b_M)_s=(R_M)_s
  \end{equation}
  \autoref{thm:signed-mlp} establishes soundness of the output zonotope in \autoref{eq:mlp-zonotope-coefficients}.
  \begin{theorem}[Fused MLP soundness]
    \label{thm:signed-mlp}
    Let $Z=\mathcal Z(c,G,L,0)$, with $L_s=[L_{s,1},\ldots,L_{s,q_s}]$, and let $M$ be the tokenwise MLP residual above.
    Suppose a nonnegative tuple $(r_N,R_N,C_g)$ satisfies the hypotheses of \autoref{thm:mlp-remainder} for every $x\in Z$.
    Then $Z_M$ defined by \autoref{eq:mlp-zonotope-coefficients} satisfies $M(Z)\subseteq Z_M$.
  \end{theorem}
  \begin{proof}[Proof]
    The Jacobian images in \autoref{eq:method-mlp-jvp} retain the input symbols, and the tokenwise MLP keeps each local image at its source token.
    The nonlinear error satisfies $|e_M|\le R_M=b_M$ by \autoref{thm:mlp-remainder}, so the four fields in \autoref{eq:mlp-zonotope-coefficients} contain every output $M(x)$ for $x\in Z$.
  \end{proof}

  \begin{remark}[Lean mechanization of fused MLP soundness]
    \label{rem:lean-signed-mlp}
    \autoref{thm:signed-mlp} is mechanized as \texttt{MlpParams.fused\_mlp\_sound} in \texttt{ZonoGpt/Mlp.lean}, and the variant with LayerNorm pre-composed as \texttt{MlpParams.fused\_mlp\_sound\_ln}.
    The mechanized theorem assumes only the three bounds of \autoref{thm:mlp-remainder} on $r_N$, $R_N$, and $C_g$.
    The mechanization also proves end-to-end soundness of the sequential MLP, covering the full transform used by the implementation (\texttt{MlpParams.Mseq\_image\_subset\_seqMlpOut} in \texttt{ZonoGpt/Sequential.lean}).
    The proof combines two facts.
    First, the JVP $J_M(c)$ is linear and acts on each token separately, so each input generator maps to exactly one output generator with the same noise symbol, and each local generator stays at its own token.
    Second, \autoref{thm:mlp-remainder} bounds the remaining error by $R_M=b_M$.
    Together, they give the symbol-preserving expansion of $M(x)$ (\texttt{MlpParams.fused\_mlp\_symbolic}), from which $M(Z)\subseteq Z_M$ follows.
    Instantiating $r_N$ and $R_N$ with the LayerNorm radii of \autoref{thm:layernorm-soundness}, and $C_g$ with $\sup|g''|$, leaves $\tau>0$ as the only hypothesis (\texttt{MlpParams.fused\_mlp\_sound\_ln\_sup}).
    The proofs mainly use the tactics \texttt{have}, \texttt{rw}, \texttt{simp only}, and \texttt{exact}, with \texttt{linarith} closing the inequalities.
  \end{remark}

  \subsubsection{MLP nonlinear-error computation}\label{app:mlp-remainder}

    \begin{proof}[Proof of \autoref{thm:mlp-remainder}]
      Fix $x\in Z$ and set $v=x-c$.
      Set $z_c=W_1\operatorname{LN}(c)+b_1$.
      The LayerNorm construction gives the following bounds (\autoref{thm:layernorm-soundness}):
      \begin{equation*}
        \operatorname{LN}(x)
        =\operatorname{LN}(c)+J_{\rm LN}(c)v+e_N,
        \qquad |J_{\rm LN}(c)v|\le r_N,
        \qquad |e_N|\le R_N
      \end{equation*}

      The change in the GELU preactivation is
      \begin{equation*}
        \begin{aligned}
          z_{\rm change}
          &=W_1\bigl(\operatorname{LN}(x)-\operatorname{LN}(c)\bigr)\\
          &=\underbrace{W_1J_{\rm LN}(c)v}_{\text{linear LayerNorm change}}
          +\underbrace{W_1e_N}_{\text{LayerNorm nonlinear error}}
        \end{aligned}
      \end{equation*}
      Therefore,
      \begin{equation*}
        |z_{\rm change}|
        \le |W_1|r_N+|W_1|R_N
        =|W_1|(r_N+R_N)
      \end{equation*}

      For tanh-GELU, set
      \begin{equation*}
        s(u)=\sqrt{2/\pi}\,(u+0.044715u^3),\qquad
        g(u)=\tfrac12u\bigl(1+\tanh s(u)\bigr)
      \end{equation*}
      The derivatives of the inner function are
      \begin{equation*}
        s'(u)=\sqrt{2/\pi}\,(1+0.134145u^2),\qquad
        s''(u)=0.26829\sqrt{2/\pi}\,u
      \end{equation*}
      Differentiating GELU once gives
      \begin{equation*}
        g'(u)
        =\tfrac12\bigl(1+\tanh s(u)\bigr)
        +\tfrac12u\operatorname{sech}^2(s(u))s'(u)
      \end{equation*}
      Differentiating again gives
      \begin{equation*}
        g''(u)
        =\operatorname{sech}^2(s(u))
        \left[
        s'(u)+\tfrac12u s''(u)
        -u\tanh(s(u))(s'(u))^2
        \right]
      \end{equation*}
      Since $|\tanh(s(u))|\le1$,
      \begin{equation*}
        |g''(u)|
        \le\operatorname{sech}^2(s(u))
        \left(
        |s'(u)|+|u|\left(|s'(u)|^2+\tfrac12|s''(u)|\right)
        \right)
      \end{equation*}
      The preactivation bound places every coordinate of $z_c+z_{\rm change}$ in
      \begin{equation*}
        [\,z_c-|W_1|(r_N+R_N),\ z_c+|W_1|(r_N+R_N)\,]
      \end{equation*}
      Bounding the preceding expression over each coordinate interval gives
      \begin{equation*}
        C_g
        =\sup_{u\in[\,z_c-|W_1|(r_N+R_N),\,
        z_c+|W_1|(r_N+R_N)\,]}|g''(u)|
      \end{equation*}

      Apply Taylor's theorem to GELU at $z_c$:
      \begin{equation*}
        g(z_c+z_{\rm change})
        =g(z_c)+g'(z_c)\odot z_{\rm change}+e_g
      \end{equation*}
      The curvature assumption on $C_g$ gives
      \begin{equation*}
        |e_g|
        \le\tfrac12C_g\odot|z_{\rm change}|^2
        \le\tfrac12C_g\odot
        \bigl(|W_1|(r_N+R_N)\bigr)^2
      \end{equation*}

      Expanding the complete MLP residual gives
      \begin{equation*}
        \begin{aligned}
          M(x)-M(c)
          &=v+W_2\bigl(g(z_c+z_{\rm change})-g(z_c)\bigr)\\
          &=v+W_2\operatorname{diag}(g'(z_c))W_1J_{\rm LN}(c)v\\
          &\quad+W_2\operatorname{diag}(g'(z_c))W_1e_N+W_2e_g
        \end{aligned}
      \end{equation*}
      The first two terms form $J_M(c)v$.
      Consequently,
      \begin{equation*}
        M(x)-M(c)-J_M(c)v
        =W_2\operatorname{diag}(g'(z_c))W_1e_N+W_2e_g
      \end{equation*}
      Taking magnitudes after composing the linear factors yields
      \begin{equation*}
        \begin{aligned}
          |M(x)-M(c)-J_M(c)v|
          &\le |W_2\operatorname{diag}(g'(z_c))W_1|R_N+|W_2||e_g|\\
          &\le |W_2\operatorname{diag}(g'(z_c))W_1|R_N\\
          &\quad+\tfrac12|W_2|
          \left(C_g\odot\bigl(|W_1|(r_N+R_N)\bigr)^2\right)
        \end{aligned}
      \end{equation*}
      The right-hand side is $R_M$ in \autoref{eq:mlp-remainder-bound}.
    \end{proof}

    \begin{remark}[Lean mechanization of the MLP nonlinear-error bound]
      \label{rem:lean-mlp-remainder}
      The computation above is mechanized as \texttt{MlpParams.mlp\_remainder} in \texttt{ZonoGpt/Mlp.lean} using a second-order Taylor bound (\texttt{taylor2\_bound} in \texttt{ZonoGpt/Calculus.lean}).
      The curvature bound $C_g$ needs to hold only on the interval centered at $z_c$ with radius $|W_1|(r_N+R_N)$.
      The mechanization defines a computable elementwise curvature bound, \texttt{MlpParams.CgSup}, and proves it valid using the global curvature estimate $|g''(u)| \le C_{\sup}$ from \texttt{ZonoGpt/GeluCalc.lean}.
      The mechanization follows the computation step by step: it writes the preactivation change as $W_1$ applied to the first-order LayerNorm change plus the LayerNorm remainder, bounds this change elementwise by $|W_1|(r_N+R_N)$, applies the second-order Taylor bound to $g$ on the segment between $z_c$ and the perturbed preactivation, and combines the two error terms into $R_M$.
      This Taylor bound requires only that $g$ and $g'$ be differentiable.
      The bound holds for every $\tau$ and for $r_N$, $R_N$, and $C_g$ of any sign.
    \end{remark}

\subsection{Fused Attention Transform}
  \label{app:attention-remainder}

  \begin{proof}[Proof of \autoref{thm:attention-radius}]
    Fix $x\in Z$ with $\Delta=x-c$.
    For each intermediate quantity $X$, its total change, linear change, and nonlinear error are
    \begin{equation}
      \delta X=X(x)-X(c),\qquad
      \dot X=J_X(c)\Delta,\qquad
      e_X=\delta X-\dot X
    \end{equation}
    Thus $\dot X$ is the linear change retained by the Jacobian, and $e_X$ is the nonlinear error to be bounded.

    First, \tool{} computes the LayerNorm linear-change and nonlinear-error bounds
    \begin{equation}
      |\dot N|\le r_N,\qquad |e_N|\le R_N
      \label{eq:method-attn-ln-error}
    \end{equation}
    using the construction in \autoref{app:layernorm}.
    The affine $QKV$ projection then gives $(f_Q,f_K,f_V)$ from $|W_{qkv}|r_N$ and $(R_Q,R_K,R_V)$ from $|W_{qkv}|R_N$.
    Hence
    \begin{equation}
      |\delta Q|\le T_Q=f_Q+R_Q,\qquad
      |\delta K|\le T_K=f_K+R_K,\qquad
      |\delta V|\le T_V=f_V+R_V
      \label{eq:method-attn-qkv-radii}
    \end{equation}

    Next, expand the score matrix $U=QK^\top/\sqrt{d_h}$.
    Its linear change and nonlinear error are
    \begin{equation}
      \begin{aligned}
        \dot U&=(\dot QK^\top+Q\dot K^\top)/\sqrt{d_h}\\
        e_U&=(e_QK^\top+Qe_K^\top+\delta Q\delta K^\top)/\sqrt{d_h}
      \end{aligned}
      \label{eq:attention-score-expansion}
    \end{equation}
    The three required bounds follow directly:
    \begin{equation}
      \begin{aligned}
        f_U&=(f_Q|K|^\top+|Q|f_K^\top)/\sqrt{d_h}\\
        R_U^N&=(R_Q|K|^\top+|Q|R_K^\top)/\sqrt{d_h}\\
        R_U^\times&=T_QT_K^\top/\sqrt{d_h}
      \end{aligned}
      \label{eq:method-attn-score-error}
    \end{equation}
    Therefore $|\dot U|\le f_U$, $|e_U|\le R_U^N+R_U^\times$, and $|\delta U|\le f_U+R_U^N+R_U^\times$.

    The score bounds must now pass through softmax.
    For every nonnegative array $r$, define
    \begin{equation}
      [\mathcal J_P(r)]_{ij}
      =P_{ij}\left(r_{ij}+\sum_kP_{ik}r_{ik}\right)
      \label{eq:method-softmax-radius-map}
    \end{equation}
    so $|J_{\rm SM}(U)\Delta U|\le\mathcal J_P(|\Delta U|)$.
    In particular, $f_P=\mathcal J_P(f_U)$ bounds $|\dot P|$.

    To bound the nonlinear softmax error, consider one row
    $p(t)=\operatorname{softmax}(U+t\delta U)$ and set
    $M=\|f_U+R_U^N+R_U^\times\|_\infty$.
    Direct differentiation gives
    \begin{equation}
      p_j''(t)=p_j(t)\left[
      (\delta U_j-\mathbb E_{p(t)}\delta U)^2
      -\operatorname{Var}_{p(t)}(\delta U)
      \right]
      \label{eq:attention-softmax-second-derivative}
    \end{equation}
    The bracket has magnitude at most $4M^2$, and $p_j(t)\le P_je^{2M}$.
    Taylor's integral remainder therefore gives
    \begin{equation}
      |\operatorname{softmax}(U+\delta U)-P-J_{\rm SM}(U)\delta U|
      \le2P\odot e^{2M}M^2
      \label{eq:attention-softmax-taylor-error}
    \end{equation}
    The retained derivative is $J_{\rm SM}(U)\dot U$.
    The omitted score error $e_U=\delta U-\dot U$ contributes at most
    $\mathcal J_P(R_U^N)+\mathcal J_P(R_U^\times)$.
    Thus the Taylor bound is
    \begin{equation}
      R_P^T=\mathcal J_P(R_U^N)+\mathcal J_P(R_U^\times)
      +2P\odot e^{2M}M^2
    \end{equation}
    Softmax outputs lie in $[0,1]$, which independently gives
    $R_P^I=\max\{P,1-P\}+f_P$.
    Consequently,
    \begin{equation}
      R_P=\min\{R_P^T,R_P^I\},\qquad T_P=f_P+R_P
      \label{eq:method-attn-softmax-error}
    \end{equation}
    satisfy $|e_P|\le R_P$ and $|\delta P|\le T_P$.

    Finally, expand the attention-value product after subtracting its center and linear change:
    \begin{equation}
      \begin{aligned}
        e_A
        &=A(x)-A(c)-J_A(c)\Delta\\
        &=W_o\operatorname{concat}\bigl(
        \underbrace{e_PV}_{\text{softmax error}}
        +\underbrace{Pe_V}_{\text{value error}}
        +\underbrace{\delta P\delta V}_{\text{product of both changes}}
        \bigr)
      \end{aligned}
      \label{eq:attention-pv-expansion}
    \end{equation}
    The three terms are bounded by
    \begin{equation}
      |e_PV|\le R_P|V|,\qquad
      |Pe_V|\le PR_V,\qquad
      |\delta P\delta V|\le(f_P+R_P)(f_V+R_V)
    \end{equation}
    because $P\ge0$ and $T_V=f_V+R_V$.
    Concatenating the heads and applying $|W_o|$ gives exactly the bound $R_A$ in \autoref{eq:attention-radius}.
    The residual skip and output bias introduce no nonlinear error because both are affine.
  \end{proof}

  \begin{remark}[Lean mechanization of the attention nonlinear-error bound]
    \label{rem:lean-attention-radius}
    \autoref{thm:attention-radius} is mechanized as \texttt{AttnParams.attention\_radius\_recipe} in \texttt{ZonoGpt/AttentionRecipe.lean}, whose only hypothesis is $\tau>0$.
    The quantity $M$ in~\eqref{eq:attention-softmax-taylor-error} must be taken per row of the score matrix (\texttt{elementwise\_M\_unsound} in \texttt{ZonoGpt/SoftmaxTaylor.lean}).
    The mechanization follows the computation above step by step: it propagates the LayerNorm radii of \autoref{thm:layernorm-soundness} through the projections via $|W|$, through the scores via the exact expansion of $QK^\top$, and through softmax via the radius map $\mathcal J_P$ and a row-wise Taylor bound based on the second derivative~\eqref{eq:attention-softmax-second-derivative}.
    Each step is an exact identity followed by the triangle inequality; the only analytic inputs are the LayerNorm bounds, which need $\tau>0$, and the softmax Taylor bound.
    The computed tuple $(f_P,R_P,f_V,R_V)$ is also proved nonnegative, so it satisfies the nonnegativity condition of \autoref{thm:attention-radius}.
    The proofs mainly use the tactics \texttt{have}, \texttt{rw}, \texttt{exact}, \texttt{simp only}, and \texttt{ring}, with \texttt{linarith} and \texttt{Finset.sum\_le\_sum} closing the inequalities.
  \end{remark}

\subsection{Fused attention soundness}
  \label{app:fused-attn-sound-proof}

  \begin{proof}[Proof of \autoref{thm:contracted-attention}]
    Fix $x\in Z$.
    At the center, $P=P(c)$ and $V=V(c)$.
    For $x\in Z$,
    \begin{equation*}
      P(x)=P+\delta P,\qquad V(x)=V+\delta V
    \end{equation*}
    The four assumptions give
    \begin{equation}
      \delta P=J_P(c)(x-c)+e_P,\qquad
      \delta V=J_V(c)(x-c)+e_V,\qquad
      |e_P|\le R_P,\qquad |e_V|\le R_V
    \end{equation}
    Expanding $(P+\delta P)(V+\delta V)$ and subtracting its center and Jacobian terms gives
    \begin{equation}
      A(x)-A(c)-J_A(c)(x-c)
      =W_o\operatorname{concat}\bigl(e_PV+Pe_V+\delta P\delta V\bigr)
    \end{equation}

    Consider the first term at query token $i$ and head $h$.
    Each row of $e_P$ sums to zero, so \autoref{eq:attention-zero-sum} gives
    \begin{equation}
      \sum_j(e_P)_{ihj}\bar v_{hj}
      =\sum_j(e_P)_{ihj}(\bar v_{hj}-a_{ih})
    \end{equation}
    For every $i,h,j$, the bound $|(e_P)_{ihj}|\le(R_P)_{ihj}$ expresses $(e_P)_{ihj}$ as a coefficient in $[-1,1]$ multiplied by $(R_P)_{ihj}$.
    Hence the generators $(R_P)_{ihj}(\bar v_{hj}-a_{ih})$ in $(L_A)_i$ contain the complete contribution of $W_o\operatorname{concat}(e_PV)$ at token $i$.

    The remaining two terms satisfy
    \begin{equation}
      |Pe_V+\delta P\delta V|
      \le P R_V+(f_P+R_P)(f_V+R_V)
    \end{equation}
    because $P\ge0$, $|\delta P|\le f_P+R_P$, and $|\delta V|\le f_V+R_V$.
    Concatenating the heads and applying $W_o$ bounds their output by
    \begin{equation}
      |W_o|\operatorname{concat}\bigl(PR_V+(f_P+R_P)(f_V+R_V)\bigr)
    \end{equation}
    which is the nonlinear-error term in $b_A$.

    It remains to place the linear term $J_A(c)(x-c)$ in the output zonotope.
    The definition of $Z$ gives
    \begin{equation}
      J_A(c)(x-c)
      =\sum_jJ_A(c)G_j\epsilon_j
      +\sum_{s=1}^{S}\sum_{k=1}^{q_s}J_A(c)\operatorname{emb}_s(L_{s,k})\epsilon^{\mathrm{loc}}_{s,k}
    \end{equation}
    The first sum is represented exactly by $(G_A)_j=J_A(c)G_j$ with the original coefficients $\epsilon_j$.
    For a local input generator at token $s$, $(L_A)_s$ retains the row at token $s$.
    At every token $i\ne s$, the triangle inequality bounds the same generator image by $|(J_A(c)\operatorname{emb}_s(L_{s,k}))_i|$.
    Summing these bounds over $s\ne i$ and $k$ gives the cross-token term in $(b_A)_i$.

    Therefore $c_A=A(c)$ represents the center, $G_A$ and $L_A$ represent the retained terms, and $b_A$ bounds all remaining terms.
    Thus every $A(x)$ belongs to $Z_A$, which proves $A(Z)\subseteq Z_A$.
  \end{proof}

  \begin{remark}[Lean mechanization of fused attention soundness]
    \label{rem:lean-contracted-attention}
    \autoref{thm:contracted-attention} is mechanized as \texttt{AttnParams.fused\_attention\_sound\_recipe} in \texttt{ZonoGpt/AttentionRecipe.lean}, whose only hypothesis is $\tau>0$.
    The argument for~\eqref{eq:attention-zero-sum}, which builds on the exact expansion \texttt{AttnParams.A\_sub\_sub\_dA}, uses the zero-sum property $\sum_j (e_P)_{ihj}=0$ to replace $\bar v_{hj}$ by $\bar v_{hj}-a_{ih}$ for any $a_{ih}$, and writes each $(e_P)_{ihj}$ as $\xi_{ihj}(R_P)_{ihj}$ with $|\xi_{ihj}|\le1$.
    The theorem requires $b=0$ for the input zonotope, because the linear expansion covers only the generator part of the input (\texttt{fused\_attention\_needs\_zero\_interval}).
    The proof separates the algebra from the bounds.
    It first writes the nonlinear error as the algebraic identity $W_o\operatorname{concat}(e_PV+Pe_V+\delta P\,\delta V)$; the zero-sum property then places the softmax term $e_PV$ into local generators, and $b_A$ bounds the remaining terms.
    This yields the symbol-preserving form (\texttt{AttnParams.fused\_attention\_symbolic\_recipe}), from which $A(Z)\subseteq Z_A$ follows for every choice of $a_{ih}$; the weighted median serves only to tighten the bound.
    The proofs mainly use the tactics \texttt{simp only}, \texttt{have}, \texttt{rw}, \texttt{exact}, and \texttt{refine}, with \texttt{funext} for equalities between sequences.
  \end{remark}

\section{Related Work}
  \label{sec:related}

  \paragraph{DNN verification tools.} State-of-the-art DNN verifiers, such as those evaluated in VNN-COMP~\citep{kaulen20256thinternationalverificationneural}, combine techniques to improve scalability.
  \textsc{$\alpha\beta$-CROWN}~\citep{zhou2024scalable,zhou2025clipandverify,zhang2022general,wang2021beta} and \textsc{NeuralSat}~\citep{duong2024harnessing,duong2025neuralsat} split verification problems and refine bounds on the resulting subproblems.
  \textsc{Marabou}~\citep{wu2024marabou,katz2017reluplex,katz2022reluplex,katz2019marabou} encodes each NNV instance as a constraint-solving problem and solves with a customized solver, while \textsc{NNEnum}~\citep{bak2021nnenum} enumerates reachable regions using star sets~\citep{tran2019star}.
  Note that every NNV tool uses some forms of abstract domains to quickly compute bounds of DNNs given input regions and thus verify problems efficiently.
  Such verifiers have also been applied beyond vision and language, \eg, to DNN-based semantic communication under generative adversarial noise~\citep{le2026verifying}.

  \paragraph{Abstraction domains.} Existing abstractions form a spectrum of speed-accuracy tradeoffs,
  \textsc{IBP}~\citep{wang2018formal} propagates interval bounds layer by layer, 
  making it the fastest method but also the loosest.
  \textsc{DeepZ}~\citep{singh2018fast} refines this by tracking affine error terms through linear layers, yielding tighter bounds at the cost of higher computational cost.
  \textsc{CROWN}~\citep{zhang2018efficient} and \textsc{DeepPoly}~\citep{singh2019abstract} apply linear relaxation with full back-substitution to the inputs, producing tighter bounds in a single forward-backward pass.
  \textsc{$\alpha$-CROWN}~\citep{xu2020automatic} extends \textsc{CROWN} by optimizing the slopes of the linear relaxations via gradient ascent, while
  \textsc{GCP-CROWN}~\citep{zhang2022general} augments \textsc{$\alpha$-CROWN} with cutting planes on intermediate neurons, pruning infeasible regions, further tightening bounds at the cost of very high computation.
  However, none of these methods scale to very large and deep networks with millions of parameters, \eg, GPT-2 Medium~\citep{radford2019language}.

  \paragraph{Transformer verification.} Existing methods adapt abstract domains to smaller Transformer models or tighten bounds for individual operations.
  \citet{shi2020robustness} verifies small models with simplified LayerNorm.
  \textsc{DeepT} extends zonotopes to Transformer operations, but evaluates deeper models with smaller configurations and restricted perturbations~\citep{bonaert2021fast}.
  \textsc{PBVerifier}~\citep{huang2026parameterized} and \citet{zhang2024galileo} improve attention-product bounds, while \citet{wei2023convex} improves softmax bounds.
  \covenn{}~\citep{duong2025compositional} scales through decomposition, whose interval bounds between components lose precision quickly.

  \paragraph{Our approach.} Prior work remains three limitations:
  (i) \textsc{DeepZ} accumulates generators with network size;
  (ii) existing work bounds softmax and attention products separately, while \covenn{}'s interval bounds lose relations between components; and
  (iii) \citet{shi2020robustness} uses simplified LayerNorm, while \textsc{DeepT} evaluates smaller configurations with restricted perturbations.
  \tool{} addresses these limitations by (i) reducing shared and local generators after each block to keep space independent of network depth (\autoref{sec:method-domain}, \autoref{sec:method-precision}, and \autoref{sec:method-complexity});
  (ii) retaining relations between components in its zonotope and bounding the attention residual as one fused transform (\autoref{sec:method-domain} and \autoref{sec:fused-attn}); and
  (iii) bounding standard LayerNorm and GELU to verify GPT-2 architectures (\autoref{sec:layernorm-fused}, \autoref{sec:gelu-bound}, and \autoref{sec:evaluation}).

\section{Benchmark Details}
  \label{app:benchmark}

  \begin{table}[h]
    \centering
    \caption{Benchmark configuration.}
    \label{tab:benchmark}
    \small
    \begin{tabular}{lll}
      \toprule
      & \textbf{MNIST} & \textbf{SST} \\
      \midrule
      Task & 10-class image classification & Binary sentiment classification \\
      Input sequence & 28 projected image rows & GPT-2 tokens, length 32 or 48 \\
      Classifier input & Mean-pooled tokens & Final token \\
      Training & Cross-entropy & IBP, from scratch \\
      Perturbation & $\ell_\infty$ over all pixels & $\ell_\infty$ over 1 or 3 token embeddings \\
      Perturbation dimension & $28\times28=784$ & $768$ or $2304$ (Small), $1024$ or $3072$ (Medium) \\
      Radii $\varepsilon$ & \multicolumn{2}{c}{$10^{-3},\ 5\!\times\!10^{-4},\ 10^{-4},\ 5\!\times\!10^{-5},\ 10^{-5}$} \\
      Instances & $8\times5\times4\times9=1440$ & $8\times5\times18\times2=1440$ \\
      \bottomrule
    \end{tabular}
  \end{table}

  Both tasks use causal GPT-2 blocks at Small (width 768, 12 heads) and Medium (width 1024, 16 heads) sizes.
  The block counts are 2, 4, 8, and 12 for Small, and 2, 4, 8, and 24 for Medium.
  For MNIST, a learned projection maps each image row to a token. We select four correctly classified test images per model and radius, clip pixel perturbations to $[0,1]$, and write one specification per alternative class.
  For SST, we perturb token embeddings rather than token IDs.
  We train the SST models from scratch with IBP, since verification-aware training improves the verifiability of the resulting networks~\citep{gowal2018effectiveness,xu2024training}.
  We select 18 correctly classified validation sentences per model and radius.
  For each sentence, we generate two specifications, perturbing either one or three token embeddings.
  All sampling uses seed 0.

  In all experiments, \tool{} bounds the combined generator pool by 512 MiB ($2^{29}$ bytes).
  The shared and local tensors store $mSd$ and $qSd$ float32 values, respectively, and therefore require $4Sd(m+q)$ bytes.
  The storage limit gives $m+q\le\lfloor 2^{29}/(4Sd)\rfloor$.
  The implementation retains at least 64 generators, giving $m+q\le\max\{64,\lfloor 2^{29}/(4Sd)\rfloor\}$.

  Let $C=\max\{64,\lfloor 2^{29}/(4Sd)\rfloor\}$ denote this combined limit, and let $\hat m$ and $\hat q$ denote the numbers of shared and local generators before each reduction.
  The values of $C$ are 6,241 and 4,681 for MNIST Small and Medium, and 5,461 and 4,096 for SST Small and Medium, respectively.
  \tool{} computes  $m$ and $q$ dynamically during propagation, while keeping $m+q\le C$ after each reduction.
  In particular, \tool{} retains $m=\min\{\hat m,C\}$ shared generators and $q=\min\{\hat q,C-m\}$ local generators per token.

\end{document}